\documentclass[10pt,twocolumn]{article}

\usepackage[letterpaper,textwidth=6.75in,textheight=9.25in,centering]{geometry}
\usepackage{amsmath,amssymb,amsfonts,amsthm,mathtools}
\usepackage{aliascnt}
\usepackage{graphicx}
\usepackage{fontawesome5}
\usepackage{tikz}
\usetikzlibrary{angles,arrows.meta,calc,decorations.pathreplacing,positioning,quotes}
\usepackage{microtype}
\usepackage{xcolor}
\usepackage{enumitem}
\usepackage{placeins}
\usepackage{cuted}
\usepackage[font=small,labelfont=bf]{caption}
\usepackage{subcaption}
\usepackage[round]{natbib}

\usepackage[colorlinks=true,citecolor=blue,linkcolor=blue,urlcolor=blue,
  hyperfootnotes=false]{hyperref}
\usepackage[nameinlink,noabbrev]{cleveref}

\newtheorem{theorem}{Theorem}[section]
\newaliascnt{lemma}{theorem}
\newtheorem{lemma}[lemma]{Lemma}
\aliascntresetthe{lemma}
\newaliascnt{proposition}{theorem}
\newtheorem{proposition}[proposition]{Proposition}
\aliascntresetthe{proposition}
\newaliascnt{corollary}{theorem}
\newtheorem{corollary}[corollary]{Corollary}
\aliascntresetthe{corollary}
\theoremstyle{definition}
\newaliascnt{definition}{theorem}

\aliascntresetthe{definition}
\newaliascnt{remark}{theorem}

\aliascntresetthe{remark}
\newaliascnt{example}{theorem}
\newtheorem{example}[example]{Example}
\aliascntresetthe{example}

\crefname{lemma}{Lemma}{Lemmas}
\crefname{proposition}{Proposition}{Propositions}
\crefname{corollary}{Corollary}{Corollaries}
\crefname{theorem}{Theorem}{Theorems}
\crefname{definition}{Definition}{Definitions}
\crefname{remark}{Remark}{Remarks}
\crefname{example}{Example}{Examples}

\AddToHook{env/theorem/begin}{\crefalias{section}{theorem}}
\AddToHook{env/lemma/begin}{\crefalias{section}{lemma}}
\AddToHook{env/proposition/begin}{\crefalias{section}{proposition}}
\AddToHook{env/corollary/begin}{\crefalias{section}{corollary}}
\AddToHook{env/definition/begin}{\crefalias{section}{definition}}
\AddToHook{env/remark/begin}{\crefalias{section}{remark}}
\AddToHook{env/example/begin}{\crefalias{section}{example}}

\newcommand{\R}{\mathbb{R}}
\newcommand{\E}{\mathbb{E}}
\newcommand{\Prob}{\mathbb{P}}
\newcommand{\Leb}{\operatorname{Leb}}
\newcommand{\Lip}{\operatorname{Lip}}
\newcommand{\esssup}{\operatorname*{ess\,sup}}
\newcommand{\ind}{\mathbf{1}}
\newcommand{\calI}{\mathcal{I}}
\newcommand{\calD}{\mathcal{D}}

\hypersetup{
  pdftitle={Sharp Root Anti-Concentration via Projective Incidence and Ordered Root Laws},
  pdfauthor={Zijun Wang; Yuchen Miao; Yifan Hu; Huanmin Liu},
  pdfsubject={Root anti-concentration for online optimization of piecewise-Lipschitz functions}
}

\title{Sharp Root Anti-Concentration via Projective Incidence and Ordered Root Laws}

\author{%
\begin{tabular}{c@{\hspace{3.5em}}c}
\begin{tabular}{c}
Zijun Wang\\
\small Northeastern University, China\\
\small\texttt{wangzj@mails.neu.edu.cn}
\end{tabular}
&
\begin{tabular}{c}
Yuchen Miao\\
\small Northeastern University, China\\
\small\texttt{miaoyc@mails.neu.edu.cn}
\end{tabular}
\\[0.9em]
\begin{tabular}{c}
Yifan Hu\\
\small Beijing Jiaotong University, China\\
\small\texttt{23251317@bjtu.edu.cn}
\end{tabular}
&
\begin{tabular}{c}
Huanmin Liu\\
\small Xiangtan University, China\\
\small\texttt{202305133221@smail.xtu.edu.cn}
\end{tabular}
\\[0.9em]
\multicolumn{2}{c}{\small
\faGithub\enspace
\href{https://github.com/sgzizi/sharp-root-anticoncentration}%
{\texttt{github.com/sgzizi/sharp-root-anticoncentration}}}
\end{tabular}}
\date{}

\begin{document}

\maketitle
\vspace{-1.5em}

\raggedbottom
\makeatletter
\AddToHookNext{shipout/after}{\global\let\@textbottom\relax}
\makeatother

\begin{abstract}
This paper answers the one-dimensional local root anti-concentration questions
posed by Balcan, Pegden, and Sharma in the context of online optimization of
piecewise-Lipschitz functions.  For a homogeneous feature curve and
coefficients whose density relative to the uniform law on a symmetric convex
body $K$ is bounded by $A$, we show that the worst-case interval-hitting
constant equals $A$ times a section-averaged projective incidence speed.  For
cube-supported coefficients, this speed is equivalent, up to universal
constants, to the projective Lipschitz constant.  This yields a sharp,
dimension-free characterization and removes the previous $\sqrt N$ loss.
For monic degree-$d$ polynomials under arbitrary coefficient laws, we prove
that the interval-hitting constant is finite if and only if the ordered
real-root laws have bounded densities, with a factor-$d$ comparison that is
sharp.  Conditional and joint coefficient-space area formulas give
coefficient-based tests for dependent and singular laws, and a two-chart bound
provides a global sufficient condition.
We apply these results to two graph-learning models.  A cost-sensitive
Gaussian-RBF harmonic classifier
uses the projective incidence theorem and achieves expected regret
$\widetilde O((An^2D e^{BD}/\ell+1)\sqrt T)$.  A common-offset
polynomial-kernel model uses rigid translation of the ordered roots and
achieves $\widetilde O((qn^2\kappa+1)\sqrt T)$ regret, even when the induced
coefficient law is singular in the ambient coefficient space.
\end{abstract}

\begingroup
\makeatletter
\renewcommand\section{\@startsection{section}{1}{\z@}%
  {-3.5ex \@plus -1ex \@minus -.2ex}%
  {0.45ex}%
  {\normalfont\Large\bfseries}}
\renewcommand\subsection{\@startsection{subsection}{2}{\z@}%
  {-0.55ex}%
  {0.35ex}%
  {\normalfont\large\bfseries}}
\makeatother
\section{Introduction}
\label{sec:introduction}

\subsection{Problem setting and open questions}
\label{sec:intro-problem}

Many data-driven algorithms depend on a real parameter and change their
behavior only when a transition equation vanishes.  Across random problem
instances, online optimization of the parameter is possible when these
transition points do not accumulate in short intervals.  The dispersion
framework makes this principle quantitative: one first represents
discontinuities by equations, controls the probability that each equation has a
root in a short interval, and then combines this local estimate with a
complexity argument to obtain a uniform dispersion and regret bound
\cite{balcan2018dispersion,balcan2020semibandit}.  Combining the local
probability estimate with boundary complexity yields the dispersion and regret
bounds in \cref{sec:applications}.
\endgroup

For the monic random polynomial
\begin{equation}
  P_\alpha(x)=x^d+\sum_{j=0}^{d-1}\alpha_jx^j,
  \qquad \alpha\in[-R,R]^d,
  \label{eq:monic-polynomial}
\end{equation}
Balcan, Dick, and Pegden proved a finite interval-hitting bound under bounded
joint density; their more general theorem treats affine images of
bounded-density latent variables under a nondegeneracy condition
\cite[Supplement, Theorem~18]{balcan2020semibandit}.  Bounded marginal densities
alone are insufficient because dependence can force a deterministic root.
Balcan, Pegden, and Sharma therefore asked for natural necessary and sufficient
conditions for finiteness, and for conditions giving polynomial dependence on
$d$ and $R$ \cite{balcan2026open}.

Their second question replaces the monomial vector by a fixed Pfaffian feature
curve $F=(F_1,\ldots,F_N)$ and considers
\begin{equation}
  \phi_\alpha(x)=\langle\alpha,F(x)\rangle,
  \qquad \alpha\in[-R,R]^N.
  \label{eq:homogeneous-family}
\end{equation}
Formal Pfaffian complexity controls zero and cell counts
\cite{balcan2025pfaffian,khovanskii1991fewnomials}, but not the probability that a
zero lands in a particular short interval.  Indeed, the degree-one family
$F(x)=(1,x/\delta)$ has fixed formal complexity and root-hitting constant of
order $1/\delta$ \cite{balcan2026open}.  The open-problem paper conjectured that
the missing analytic condition is bounded motion of the normalized curve
\begin{equation}
  G(x)=\frac{F(x)}{\lVert F(x)\rVert_2}.
  \label{eq:normalized-curve-intro}
\end{equation}

Coefficient dependence can concentrate a root, whereas fixed Pfaffian format
can conceal arbitrarily poor geometric conditioning.  These are the two open
questions posed by Balcan, Pegden, and Sharma at COLT 2026
\cite[Open Questions~1--2]{balcan2026open}:

\begin{center}
\begin{minipage}{0.92\columnwidth}
\centering\itshape
Which natural conditions on dependent coefficient laws are necessary and
sufficient for finite, or polynomially bounded, interval root hitting for
random polynomials, and what normalization yields the corresponding guarantee
for Pfaffian linear combinations?
\end{minipage}
\end{center}

\subsection{Answers to Open Questions 1--2}
\label{sec:intro-answers}

For Open Question~1, the interval-hitting constant is finite exactly when the
ordered real-root pushforwards have uniformly bounded densities.  More
precisely, let $M_\calD(\Theta)$ denote the supremum of their density bounds
over laws in $\calD$.  \Cref{thm:root-coordinates,cor:root-coordinate-class}
prove
\[
 M_\calD(\Theta)\le C_\calD(\Theta)\le dM_\calD(\Theta),
\]
and \cref{prop:factor-d-sharp} shows that the factor $d$ is sharp.  It follows
that, for classes $\calD_{d,R}$, the interval-hitting constant is polynomially
bounded in $(d,R)$ exactly when $M_{\calD_{d,R}}$ is.  The coefficient-space
formulas in
\cref{prop:conditional-rice,prop:joint-rice} express the criterion in
coefficient coordinates.  For conditional-density models,
\cref{cor:conditional-class} is necessary and sufficient, while
\cref{thm:two-chart} gives a global sufficient bound.  In particular,
\cref{cor:joint-density-recovery} recovers the
polynomial bound stated as Theorem~2 of \citet{balcan2026open}.

For Open Question~2, \cref{thm:body-characterization} identifies the exact
support-adapted quantity.  If the coefficient law has density ratio at most
$A$ with respect to the uniform law on a symmetric convex body $K$, then
\[
 C_{K,A}(F;\Theta)=A\Lambda_K(F;\Theta),
\]
where $\Lambda_K$ averages projective velocity over central sections of $K$.
For cube-supported coefficients, \cref{cor:cube-comparison} proves the
dimension-free comparison
\[
 c_{\mathrm{cube}} A L_{\mathrm P}
 \le C_{R,\kappa}=A\Lambda_{Q_R}
 \le C_{\mathrm{cube}} A L_{\mathrm P}.
\]
Thus common zeros and failure of projective Lipschitz regularity are exactly
the obstructions, and the previous $\sqrt N$ loss disappears.  The affine
formulation in \cref{thm:affine} covers fixed-leading-coefficient families,
while \cref{ex:exponential-pfaffian} gives an explicit nonpolynomial Pfaffian
instance.

\subsection{Graph-learning applications}
\label{sec:intro-applications}

\Cref{sec:applications} proves dispersion and regret bounds for two
graph-learning models.  The first
tunes the inverse bandwidth of a Gaussian-RBF graph for cost-sensitive
harmonic classification.  Its transition feature is the nonpolynomial curve
$(s_v(\gamma),1-s_v(\gamma))$; a grounded-Laplacian derivative estimate and
the exact coefficient-body formula give local anti-concentration, while the
Chebyshev zero bound for exponential polynomials
(\cref{lem:exp-poly-zeros}) gives
$\log K=O(n\log n)$ boundary complexity without discretizing the graph
dissimilarities.  The resulting expected regret is
$\widetilde O((An^2D e^{BD}/\ell+1)\sqrt T)$.  The second treats the
polynomial-kernel problem of \citet{balcan2021ssl} under a task-level
calibration offset shared by all similarities.  Here the induced coefficient
law is supported on a one-dimensional image in the ambient coefficient space,
but the ordered roots translate rigidly with the offset, yielding
$\widetilde O((qn^2\kappa+1)\sqrt T)$ expected regret.

\subsection{Related work}
\label{sec:related-work}

\paragraph{Random polynomials and area formulas.}
Classical polynomial small-ball inequalities control the value of a polynomial
of random inputs under assumptions such as log-concavity or independence
\cite{carbery2001distributional,meka2016anticoncentration}.  The event studied
here is
different: a random coefficient vector must produce \emph{some} parameter root
inside a prescribed interval.  Expected real-root counts and their
concentration are studied under many structured coefficient models
\cite{aguirre2025concentration,do2024stronglaw,edelman1995zeros}, including
dependent Gaussian coefficients \cite{matayoshi2012dependent}.  Root-intensity
identities are instances of the Rice and area formulas developed systematically
in \cite{azais2009level}; the incidence proof also uses the geometric area and
coarea formulas \cite{federer1969geometric}.  Normalized moment curves have a
classical geometric role for Gaussian zero densities
\cite{edelman1995zeros}.  Related boundary-crossing estimates appear in
noninteractive one-bit mean estimation, where shifted random grids support
multiscale refinement under finite-moment assumptions
\cite{miao2026universalrefinementinteractionorderoptimal}.

\paragraph{Projective geometry.}
Projective geometry and condition numbers are treated broadly in
\cite{burgisser2013condition}.  The exact coefficient-body formula comes from
the swept volume of central hyperplane sections.  For a cube, two-dimensional
isotropic log-concave marginal bounds
\cite[Lemma~2(b,c)]{balcan2013linear} compare this swept volume with projective
angle uniformly in dimension; Ball's slicing theorem \cite{ball1986cube}
gives the affine slab bound.

\paragraph{Graph learning.}
For graph-based semi-supervised learning, \citet{balcan2021ssl} establish
online guarantees for polynomial kernels when all input similarities have a
bounded joint density and identify Gaussian-RBF graph families as a standard
configuration problem.  Cost-sensitive label propagation has been studied in
response to unequal error costs \cite{wan2019cost}.  The applications in
\cref{sec:applications} use projective incidence for the RBF family and
ordered-root translation for a polynomial-kernel law that is singular in the
ambient coefficient space.

\section{Setup and projective incidence}
\label{sec:setup}

Throughout, $\Theta\subseteq\R$ is a nondegenerate interval and
$\calI(\Theta)$ is the
collection of bounded, nondegenerate \emph{closed} intervals contained in
$\Theta$.  This convention ensures that roots at boundary points of $\Theta$
are detected.  Fix $d\ge1$.  For the monic polynomial
\cref{eq:monic-polynomial}, let
\[
 Z_\alpha(B)=\#\{x\in B:P_\alpha(x)=0\}
\]
count \emph{distinct} real roots in a Borel set $B$.  For a coefficient law
$\mu$, define
\begin{equation}
 \begin{aligned}
 c_\mu(\Theta)
 &=\sup_{I\in\calI(\Theta)}
 \frac{\Prob_{\alpha\sim\mu}(Z_\alpha(I)\ge1)}{|I|},\\
 C_\calD(\Theta)&=\sup_{\mu\in\calD}c_\mu(\Theta).
 \end{aligned}
 \label{eq:polynomial-hitting-constant}
\end{equation}
The values may be infinite.

For $N\ge1$, $R>0$, and $\kappa>0$, let
$\calD^N_{R,\kappa}$ consist of probability laws supported on
$Q_R=[-R,R]^N$ with Lebesgue density bounded by $\kappa$.  Put
\begin{equation}
 A=(2R)^N\kappa.
 \label{eq:normalized-density}
\end{equation}
The class is nonempty only if $A\ge1$.  For a fixed continuous
$F:\Theta\to\R^N$, define $c_\mu(F;\Theta)$ by replacing the root event in
\cref{eq:polynomial-hitting-constant} with
$\{\exists x\in I:\langle\alpha,F(x)\rangle=0\}$, and let
$C_{R,\kappa}(F;\Theta)=\sup_{\mu\in\calD^N_{R,\kappa}}c_\mu(F;\Theta)$.

The coefficient-body formulation will also be useful.  Let
$K\subset\R^N$ be a full-dimensional compact origin-symmetric convex body
and let $A\ge1$.
Write $\calD_{K,A}$ for the probability laws supported on $K$ whose
Lebesgue densities are at most
$A/\operatorname{vol}_N(K)$, and set
\begin{equation}
 C_{K,A}(F;\Theta)
 =\sup_{\mu\in\calD_{K,A}}c_\mu(F;\Theta).
 \label{eq:body-hitting-constant}
\end{equation}
Thus $\calD^N_{R,\kappa}=\calD_{Q_R,A}$ for the value of $A$ in
\cref{eq:normalized-density}.

The homogeneous equation sees only the line spanned by $F(x)$.  For nonzero
$u,v\in\R^N$, define the real-projective distance
\begin{equation}
 d_{\mathrm{P}}([u],[v])
 =\arccos\!\left(
   \frac{|\langle u,v\rangle|}{\lVert u\rVert_2\lVert v\rVert_2}
 \right)\in[0,\pi/2].
 \label{eq:projective-distance}
\end{equation}
For a nonvanishing $F$, put
\begin{equation}
 L_{\mathrm{P}}(F;\Theta)
 =\sup_{x\ne y\in\Theta}
 \frac{d_{\mathrm{P}}([F(x)],[F(y)])}{|x-y|}\in[0,\infty].
 \label{eq:projective-lipschitz}
\end{equation}
Unlike the Euclidean Lipschitz constant of a chosen normalization, this
quantity is invariant under every nonzero pointwise rescaling of $F$.
For unit vectors, let the oriented angle be
$\angle(u,v)=\arccos\langle u,v\rangle\in[0,\pi]$.

\subsection{The coefficient-body formula}
\label{sec:projective}

The scaling example from \cite{balcan2026open} identifies the obstruction.
For $F_\delta(x)=(1,x/\delta)$ on $[-1,1]$, formal polynomial and Pfaffian
degrees are fixed, whereas
\begin{equation}
 \left\lVert\frac{d}{dx}
 \frac{F_\delta(x)}{\lVert F_\delta(x)\rVert_2}\right\rVert_2
 =\frac{\delta^{-1}}{1+(x/\delta)^2}
 \label{eq:scaled-speed}
\end{equation}
has supremum $1/\delta$.  The root-hitting lower bound in
\cite{balcan2026open} has the same $1/\delta$ scale.  The normalized curve
therefore detects precisely what formal degree misses.

For a $C^1$ nonvanishing curve $F$, define its
\emph{$K$-incidence speed} at $x\in\Theta^\circ$ by
\begin{equation}
 \begin{aligned}
 \lambda_K(F;x)
 &=\frac{1}{\operatorname{vol}_N(K)\lVert F(x)\rVert_2}\\
 &\quad\times\int_{K\cap F(x)^\perp}
  |\langle a,F'(x)\rangle|\,
  d\mathcal H^{N-1}(a),
 \end{aligned}
 \label{eq:body-incidence-speed}
\end{equation}
and put
\begin{equation}
\Lambda_K(F;\Theta)
 =\operatorname*{ess\,sup}_{x\in\Theta^\circ}\lambda_K(F;x).
 \label{eq:max-body-incidence-speed}
\end{equation}
For $N=1$, the integral in \cref{eq:body-incidence-speed} uses counting
measure $\mathcal H^0$; since $K\cap F(x)^\perp=\{0\}$, the incidence speed
is zero.
This is a support-body-induced Finsler speed on projective space.  Indeed, for
a unit direction $u$ and a tangent vector $v\in u^\perp$, set
\begin{equation}
 \lVert v\rVert_{K,[u]}
 =\frac{1}{\operatorname{vol}_N(K)}
   \int_{K\cap u^\perp}|\langle a,v\rangle|\,
   d\mathcal H^{N-1}(a).
 \label{eq:body-projective-finsler}
\end{equation}
With $G=F/\lVert F\rVert_2$,
$\lambda_K(F;x)=\lVert G'(x)\rVert_{K,[G(x)]}$.  Thus the radial part of
$F'(x)$ is discarded automatically.  The speed is unchanged by a nonzero
$C^1$ pointwise rescaling of $F$ or by a homothetic rescaling of $K$.
Unlike a support-function bound, it records both the displacement of the
moving hyperplane and the size of the section through which it moves.

\begin{lemma}[Incidence area and first variation]
\label{lem:incidence-area}
Let $K$ be a full-dimensional origin-symmetric convex body and let
$F\in C^1(\Theta;\R^N)$ be nonvanishing.  For
\[
 E_F(I)=\{a\in K:\exists x\in I,\ \langle a,F(x)\rangle=0\},
\]
every $I\in\calI(\Theta)$ satisfies
\begin{equation}
 \frac{\operatorname{vol}_N(E_F(I))}
      {\operatorname{vol}_N(K)}
 \le\int_I\lambda_K(F;x)\,dx.
 \label{eq:incidence-area-bound}
\end{equation}
Moreover, for almost every $x\in\Theta^\circ$, let
\[
 D_{x,h}=\{a\in K:
   \langle a,F(x)\rangle\langle a,F(x+h)\rangle<0\}.
\]
Then
\begin{equation}
 \lim_{h\downarrow0}
 \frac{\operatorname{vol}_N(D_{x,h})}
      {h\,\operatorname{vol}_N(K)}
 =\lambda_K(F;x).
 \label{eq:incidence-first-variation}
\end{equation}
\end{lemma}

\begin{theorem}[Coefficient-body formula]
\label{thm:body-characterization}
Let $K\subset\R^N$ be a full-dimensional compact origin-symmetric convex
body, let $A\ge1$,
and let $F\in C^1(\Theta;\R^N)$.  If $F$ is nonvanishing, then
\begin{equation}
 C_{K,A}(F;\Theta)=A\Lambda_K(F;\Theta).
 \label{eq:body-characterization}
\end{equation}
If instead $F(x_*)=0$ for some $x_*\in\Theta$, then
$C_{K,A}(F;\Theta)=\infty$.
\end{theorem}

\noindent\textbf{Proof sketch.}\enspace
The root-producing coefficient vectors over $I$ form the projection of the
incidence manifold
\[
 \{(a,x)\in K\times I:\langle a,F(x)\rangle=0\}.
\]
The area formula gives \cref{eq:incidence-area-bound}; multiplying by the
density cap $A/\operatorname{vol}_N(K)$ proves the upper bound.  For the
reverse inequality, take a short interval $[x,x+h]$.  Endpoint sign
disagreement forces a root, and
\cref{eq:incidence-first-variation} gives its uniform-$K$ probability to
first order.  A density equal to $A/\operatorname{vol}_N(K)$ on this
disagreement set and constant on its complement is admissible for all
sufficiently small $h$.  It multiplies that probability by $A$.  Letting
$h\downarrow0$ and taking the essential supremum gives equality.  A common
zero is a deterministic root.

\begin{strip}
\noindent\begin{minipage}{\textwidth}
\centering
\captionsetup{type=figure,hypcap=false}
\includegraphics[width=\textwidth]{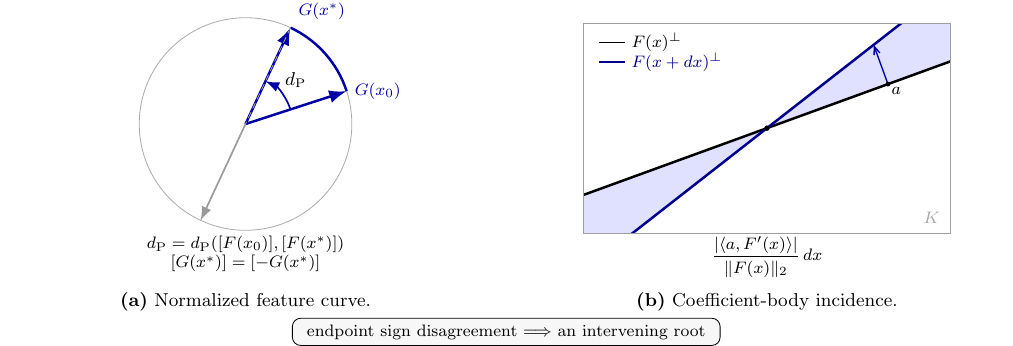}
\caption{Projective motion and coefficient-body incidence.  The homogeneous
equation depends only on $[F(x)]$.  As this class moves, the central
hyperplane $F(x)^\perp$ sweeps through $K$; integrating its local swept
thickness gives the incidence speed in
\cref{eq:body-incidence-speed}.  Conversely, endpoint sign disagreement
forces an intervening root.}
\label{fig:projective-slab}
\end{minipage}
\end{strip}

\begin{corollary}[Dimension-free cube comparison]
\leavevmode\label{cor:cube-comparison}%
Let $F\in C^1(\Theta;\R^N)$ be nonvanishing and suppose
$\calD^N_{R,\kappa}$ is nonempty.  There are universal constants
$0<c_{\mathrm{cube}}\le C_{\mathrm{cube}}<\infty$ such that
\begin{equation}
 \begin{aligned}
 c_{\mathrm{cube}}A L_{\mathrm P}(F;\Theta)
 &\le C_{R,\kappa}(F;\Theta)
  =A\Lambda_{Q_R}(F;\Theta)\\
 &\le C_{\mathrm{cube}}A L_{\mathrm P}(F;\Theta).
 \end{aligned}
 \label{eq:projective-characterization}
\end{equation}
Consequently,
\begin{equation}
 \begin{aligned}
 &C_{R,\kappa}(F;\Theta)<\infty\\
 &\quad\Longleftrightarrow\quad L_{\mathrm P}(F;\Theta)<\infty.
 \end{aligned}
 \label{eq:projective-iff}
\end{equation}
Together with \cref{thm:body-characterization}, a common zero instead makes
the constant infinite.  Among nonvanishing parameterized families, the
root-hitting constants are polynomially bounded exactly when
$A L_{\mathrm P}$ is polynomially bounded.
\end{corollary}

The factor $\sqrt N$ from the enclosing-slab argument is therefore absent,
and the dependence on both $A$ and projective speed is sharp.  To see the
dimension-free comparison, scale the uniform cube to isotropic position.
Every two-dimensional marginal is isotropic log-concave.  Standard
fixed-dimensional density bounds show that the section integral in
\cref{eq:body-incidence-speed} lies between two universal constants times
the instantaneous projective speed.  Taking essential suprema and applying
\cref{thm:body-characterization} proves
\cref{eq:projective-characterization}.

\begin{corollary}[A compact-domain derivative bound]
\label{cor:compact-dichotomy}
Let $J$ be a nondegenerate compact interval, let $F:J\to\R^N$ be $C^1$ and
nonvanishing, and suppose $\calD^N_{R,\kappa}$ is nonempty.  Then
\begin{equation}
 C_{R,\kappa}(F;J)
 \le C_{\mathrm{cube}}A\,
       \frac{\sup_{x\in J}\lVert F'(x)\rVert_2}
            {\inf_{x\in J}\lVert F(x)\rVert_2}.
 \label{eq:compact-bound}
\end{equation}
\end{corollary}

{\clubpenalty=10000
The theorem uses no Pfaffian-specific property.  Pfaffian definability is
needed in the representation and boundary-complexity steps of a dispersion
argument, whereas \cref{cor:cube-comparison} gives the local
probability condition.  It also yields the requested polynomial statement for
any parameterized family of nonvanishing curves: whenever
$A=(2R)^N\kappa$ is polynomially controlled,
$C_{R,\kappa}(F;\Theta)$ is polynomially bounded if and only if
$L_{\mathrm{P}}(F;\Theta)$ is.  Controlling $R$ and $\kappa$ separately is
not enough when $N$ varies; the likelihood-ratio parameter $A$ is the relevant
density scale.
\par}

\begin{example}[A nonpolynomial Pfaffian curve]
\label{ex:exponential-pfaffian}
Fix $B>0$ and $\lambda\in\R$, and let
\[
 F_\lambda(x)=(1,e^{\lambda x}),\qquad x\in[-B,B].
\]
This is a nonpolynomial Pfaffian feature curve when $\lambda\ne0$.  With
$\psi(x)=\arctan(e^{\lambda x})$, its normalized projective representative is
$(\cos\psi(x),\sin\psi(x))$.  Since $\psi([-B,B])\subset(0,\pi/2)$,
$d_{\mathrm P}([F_\lambda(x)],[F_\lambda(y)])=|\psi(x)-\psi(y)|$, and
\[
 |\psi'(x)|
 =\frac{|\lambda|e^{\lambda x}}{1+e^{2\lambda x}}
 \le\frac{|\lambda|}{2}.
\]
Equality holds at $x=0$, so
$L_{\mathrm P}(F_\lambda;[-B,B])=|\lambda|/2$.  Whenever
$\calD^2_{R,\kappa}$ is nonempty, equivalently
$A=(2R)^2\kappa\ge1$,
\cref{cor:cube-comparison} gives
\[
 \frac{c_{\mathrm{cube}}A|\lambda|}{2}
 \le C_{R,\kappa}(F_\lambda;[-B,B])
 \le\frac{C_{\mathrm{cube}}A|\lambda|}{2}.
\]
\end{example}

The scaled curve in \cref{eq:scaled-speed} has
$L_{\mathrm{P}}=1/\delta$: on a small neighborhood of zero, the displayed
normalized speed is arbitrarily close to $1/\delta$.  Thus the
$1/\delta$ lower bound in \cite{balcan2026open} and
\cref{eq:projective-characterization} agree in
scale.  In fact, \cref{thm:body-characterization} makes the calculation
exact.  For $N\ge2$, $A\ge1$, $K=Q_R$, $\delta\in(0,1]$, and
$F_{\delta,N}(x)=(1,x/\delta,0,\ldots,0)$,
\begin{equation}
 \begin{aligned}
 \Lambda_{Q_R}(F_{\delta,N};[-1,1])&=\frac{1}{4\delta},\\
 C_{Q_R,A}(F_{\delta,N};[-1,1])&=\frac{A}{4\delta}.
 \end{aligned}
 \label{eq:scaled-exact-sharpness}
\end{equation}
Thus the $A L_{\mathrm P}$ scale is attained, already in two active
coefficient dimensions, while formal Pfaffian degree remains fixed.

\section{Polynomial root intensity and conditional density}
\label{sec:intensity}

For an arbitrary coefficient law $\mu$ in \cref{eq:monic-polynomial}, define
the expected distinct-root measure
\begin{equation}
 \Lambda_\mu(B)=\E_{\alpha\sim\mu}Z_\alpha(B),
 \qquad B\subseteq\Theta\ \text{Borel}.
 \label{eq:expected-root-measure}
\end{equation}
It is a finite Borel measure of total mass at most $d$.  No density or
independence assumption is made.

\begin{theorem}[Root-intensity characterization]
\label{thm:intensity}
For every coefficient law $\mu$, the following are equivalent:
\begin{enumerate}[label=(\alph*),leftmargin=*]
\item $c_\mu(\Theta)<\infty$;
\item $\Lambda_\mu$ is absolutely continuous with respect to Lebesgue measure
on $\Theta$ and has an essentially bounded density $\rho_\mu$.
\end{enumerate}
Moreover,
\begin{equation}
 c_\mu(\Theta)\le\lVert\rho_\mu\rVert_{L^\infty(\Theta)}
 \le d\,c_\mu(\Theta).
 \label{eq:intensity-sandwich}
\end{equation}
\end{theorem}

\paragraph{Proof sketch.}
For every coefficient vector, the distinct-root count lies between the
indicator of a hit and $d$ times that indicator.  After expectation, this
sandwich proves one implication directly.  In the other direction, interval
domination extends to Borel domination by exhaustion and outer regularity, and
the Radon--Nikodym theorem produces the bounded density.

For a class $\calD$, interpret
$\lVert\rho_\mu\rVert_\infty=\infty$ when the density does not exist.  Taking
suprema in \cref{eq:intensity-sandwich} gives
\begin{equation}
 C_\calD(\Theta)
 \le \sup_{\mu\in\calD}\lVert\rho_\mu\rVert_\infty
 \le d C_\calD(\Theta).
 \label{eq:class-intensity}
\end{equation}
Hence $C_\calD$ is finite exactly when the expected root measures are uniformly
dominated by a bounded multiple of Lebesgue measure.  For a class family
indexed by $(d,R)$, polynomial boundedness of the middle quantity in
\cref{eq:class-intensity} is equivalent to polynomial boundedness of the
hitting constant, because the loss is only the factor $d$.  The same condition
can be expressed through the ordered-root pushforwards below.

For each coefficient vector, order the distinct roots in $\Theta$ as
\[
 r_1(\alpha)<\cdots<r_{m(\alpha)}(\alpha),\qquad m(\alpha)\le d.
\]
These partial root maps are Borel by Lemma~\ref{lem:root-kernel-measurable}.  For
$1\le k\le d$, define the subprobability measure
\begin{equation}
 \nu_{\mu,k}(B)
 =\Prob_{\alpha\sim\mu}\bigl(m(\alpha)\ge k,\ r_k(\alpha)\in B\bigr).
 \label{eq:ordered-root-law}
\end{equation}
If every $\nu_{\mu,k}$ has a bounded Lebesgue density $p_{\mu,k}$ on
$\Theta$, put
\[
 M_\mu(\Theta)=\max_{1\le k\le d}
 \lVert p_{\mu,k}\rVert_{L^\infty(\Theta)};
\]
otherwise put $M_\mu(\Theta)=\infty$.

\begin{theorem}[Ordered-root characterization]
\label{thm:root-coordinates}
For every coefficient law $\mu$, including a singular law,
\begin{equation}
 M_\mu(\Theta)\le c_\mu(\Theta)\le dM_\mu(\Theta).
 \label{eq:root-coordinate-sandwich}
\end{equation}
In particular, $c_\mu$ is finite exactly when all ordered-root pushforwards
have bounded densities.
\end{theorem}

\paragraph{Proof sketch.}
Each ordered-root event is contained in the root-hitting event, while the
root-hitting event is the union of at most $d$ ordered-root events.  Applying
interval-to-Borel domination to the first inclusion and a union bound to the
second gives the two sides of \cref{eq:root-coordinate-sandwich}.  Borel
measurability follows from \cref{lem:root-kernel-measurable}.

\begin{proposition}[Sharpness of the factor $d$]
\label{prop:factor-d-sharp}
For every $d\ge1$ and every $R>0$, there is a law supported on
$[-R,R]^d$ for which
\[
 c_\mu(\R)=dM_\mu(\R)<\infty.
\]
The law may be chosen to have a bounded Lebesgue density.
Thus the factor $d$ is sharp even within the bounded-coefficient model of
Open Question~1.
\end{proposition}

\begin{proof}
Put
\[
 r=\frac{\min\{R,1\}}2,\qquad \varepsilon=\frac rd.
\]
Set $\ell=\varepsilon/[4(d(d-1)+1)]$ and place
$d(d-1)+1$ pairwise disjoint open intervals of length $\ell$ inside
$(-\varepsilon,\varepsilon)$, with adjacent intervals separated by at least
$\ell$.  Designate the centered interval
$J_*=(-\ell/2,\ell/2)$.  Index the remaining intervals as
$J_{k,j}$, $k\ne j$, so that for each $k$ the intervals with $j<k$ lie to the
left of $J_*$ in increasing order and those with $j>k$ lie to its right in
increasing order.  There are exactly $d(d-1)/2$ auxiliary intervals on each
side.

Draw $K$ uniformly from $\{1,\ldots,d\}$.  Given $K=k$, draw the variables
$X_j$ independently, with $X_k\sim\operatorname{Unif}(J_*)$ and
$X_j\sim\operatorname{Unif}(J_{k,j})$ for $j\ne k$.  The ordering of the
intervals ensures $X_1<\cdots<X_d$ almost surely.  Set
$P(x)=\prod_{j=1}^d(x-X_j)$.  The $j$th ordered-root density equals
$1/(d\ell)$ on $J_*$ and on its $d-1$ auxiliary intervals, so
$M_\mu=1/(d\ell)$.  The aggregate root density is $1/\ell$ on $J_*$,
$1/(d\ell)$ on the auxiliary intervals, and zero elsewhere.  Hence every
interval $I$ has hitting probability at most
$\E Z_\alpha(I)\le |I|/\ell$.  Equality holds for every closed subinterval of
$J_*$, because each polynomial has exactly one root there and that root is
uniform.  Therefore $c_\mu=1/\ell=dM_\mu$.

It remains to verify the prescribed coefficient support.  If
\[
 P(x)=x^d+\sum_{j=0}^{d-1}\alpha_jx^j,
 \]
then Vi\`ete's formulas and $|X_j|<\varepsilon$ give, for $1\le q\le d$,
\[
 |\alpha_{d-q}|
 \le \binom{d}{q}\varepsilon^q
 \le\frac{(d\varepsilon)^q}{q!}
 =\frac{r^q}{q!}
 \le R.
\]
Thus $\alpha\in[-R,R]^d$ almost surely.  Finally, on each conditional root
box the Vi\`ete map from ordered roots to coefficients is injective and has
Jacobian
\[
 \prod_{1\le i<j\le d}|X_j-X_i|\ge\ell^{d(d-1)/2}.
\]
The change-of-variables formula, applied to the finite mixture over $K$, shows
that the induced coefficient law has a density bounded by
$\ell^{-d-d(d-1)/2}$ (and in particular finite).
\Cref{fig:factor-d-sharp} depicts the interval arrangement for $d=4$.
\end{proof}

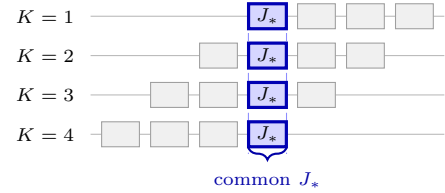
\begin{figure}[!htbp]
\centering
\begin{tikzpicture}[
  x=.72cm,y=.52cm,
  font=\scriptsize,
  auxiliary/.style={draw=gray!75,fill=gray!12,minimum width=5mm,
                    minimum height=3.3mm,inner sep=0pt},
  common/.style={draw=blue!70!black,fill=blue!16,very thick,
                 minimum width=5mm,minimum height=3.3mm,inner sep=0pt}
]
  \draw[blue!45,densely dashed] (-.35,3.25)--(-.35,-.50);
  \draw[blue!45,densely dashed] ( .35,3.25)--( .35,-.50);
  \foreach \y/\k in {3/1,2/2,1/3,0/4}{
    \draw[gray!55] (-3.25,\y)--(3.25,\y);
    \node[anchor=east] at (-3.38,\y) {$K=\k$};
    \node[common] at (0,\y) {$J_*$};
  }
  \node[auxiliary] at (.90,3) {};
  \node[auxiliary] at (1.80,3) {};
  \node[auxiliary] at (2.70,3) {};

  \node[auxiliary] at (-.90,2) {};
  \node[auxiliary] at (.90,2) {};
  \node[auxiliary] at (1.80,2) {};

  \node[auxiliary] at (-1.80,1) {};
  \node[auxiliary] at (-.90,1) {};
  \node[auxiliary] at (.90,1) {};

  \node[auxiliary] at (-2.70,0) {};
  \node[auxiliary] at (-1.80,0) {};
  \node[auxiliary] at (-.90,0) {};

  \draw[decorate,blue!65!black,thick,
        decoration={brace,mirror,amplitude=4.5pt,raise=1.5pt}]
    (-.35,-.24)--(.35,-.24)
    node[midway,below=7pt] {common $J_*$};
\end{tikzpicture}
\caption{Sharpness of the ordered-root factor, illustrated for $d=4$.
Conditioned on $K=k$, the common interval $J_*$ supports the $k$th ordered
root; the other roots lie in pairwise disjoint auxiliary intervals arranged
on the appropriate side of $J_*$.  Thus
$M_\mu=1/(d\ell)$, $c_\mu=1/\ell$, and $c_\mu=dM_\mu$.}
\label{fig:factor-d-sharp}
\end{figure}

\FloatBarrier

\begin{corollary}[Distribution-class criterion]
\label{cor:root-coordinate-class}
For a class $\calD$, let
\[
 M_\calD(\Theta)=\sup_{\mu\in\calD}M_\mu(\Theta).
\]
Then
\begin{equation}
 M_\calD(\Theta)\le C_\calD(\Theta)\le dM_\calD(\Theta).
 \label{eq:root-coordinate-class}
\end{equation}
Consequently, $C_\calD$ is finite if and only if the ordered-root
pushforwards are uniformly bounded-density laws.
\end{corollary}

For the polynomial-growth clause of Open Question~1, let
$\{\calD_{d,R}\}_{d,R}$ be a family of coefficient-law classes supported on
$[-R,R]^d$.  Applying
\cref{eq:root-coordinate-class} at each $(d,R)$ gives
\[
 M_{\calD_{d,R}}(\Theta)
 \le C_{\calD_{d,R}}(\Theta)
 \le dM_{\calD_{d,R}}(\Theta).
\]
Since $d$ is itself an allowed polynomial factor,
$C_{\calD_{d,R}}$ is polynomially bounded in $(d,R)$ if and only if
$M_{\calD_{d,R}}$ is.

The ordered-root laws characterize finiteness through fixed Borel functions of
the coefficient vector.  Joint-density, conditional-density, affine-image, and
singular models can therefore be analyzed in coefficient coordinates.
Equivalently,
$\Lambda_\mu=\sum_{k=1}^d\nu_{\mu,k}$, so
$M_\mu\le\lVert\rho_\mu\rVert_\infty\le dM_\mu$ whenever the densities exist.

Propositions~\ref{prop:conditional-rice} and \ref{prop:joint-rice} give
conditional and joint coefficient-space forms of the Kac--Rice/area formula.
The conditional form permits singular dependence among the remaining
coefficients.

Let $A_-=(\alpha_1,\ldots,\alpha_{d-1})$ and, for
$a=(a_1,\ldots,a_{d-1})$, set
\begin{equation}
 \begin{aligned}
 h_a(x)&=-x^d-\sum_{j=1}^{d-1}a_jx^j,\\
 |h_a'(x)|&=\left|dx^{d-1}+\sum_{j=1}^{d-1}ja_jx^{j-1}\right|.
 \end{aligned}
 \label{eq:conditional-map}
\end{equation}

\begin{proposition}[Conditional root-intensity formula]
\label{prop:conditional-rice}
Suppose a regular conditional law of $\alpha_0$ given $A_-=a$ admits a jointly
measurable density $q(u\mid a)$, where $\nu$ is the law of $A_-$.  Then
$\Lambda_\mu$ is absolutely continuous and, as an identity in
$L^1_{\mathrm{loc}}(\Theta)$ and hence for almost every $x$,
\begin{equation}
 \rho_\mu(x)=
 \int |h_a'(x)|q(h_a(x)\mid a)\,d\nu(a).
 \label{eq:conditional-rice}
\end{equation}
The equivalence class of the right-hand side is independent of the chosen
version of $q$.
\end{proposition}

\begin{proof}[Proof sketch]
Condition on $A_-=a$.  A value $u$ of $\alpha_0$ produces a root at $x$ exactly
when $u=h_a(x)$.  The one-dimensional area formula, applied to $h_a$ on compact
subintervals and then exhausted over $\Theta$, gives
\[
 \int q(u\mid a)Z_{(u,a)}(B)\,du
 =\int_B |h_a'(x)|q(h_a(x)\mid a)\,dx.
\]
Integrating in $a$ and using Tonelli proves the formula; the same area-formula
identity proves version independence.
\end{proof}

\begin{corollary}[Conditional-density class criterion]
\label{cor:conditional-class}
Let $\calD$ be a class of coefficient laws for which the conditional densities
in \cref{prop:conditional-rice} exist.  For $\mu\in\calD$, write
\[
 H_\mu(\Theta)
 =\left\|
 x\longmapsto
 \int |h_a'(x)|q_\mu(h_a(x)\mid a)\,d\nu_\mu(a)
 \right\|_{L^\infty(\Theta)}
\]
and put $H_\calD(\Theta)=\sup_{\mu\in\calD}H_\mu(\Theta)$.  Then
\begin{equation}
 C_\calD(\Theta)
 \le H_\calD(\Theta)
 \le dC_\calD(\Theta).
 \label{eq:conditional-class-criterion}
\end{equation}
Consequently, $C_\calD(\Theta)<\infty$ if and only if
$H_\calD(\Theta)<\infty$.  For any indexed family
$\{\calD_{d,R}\}_{d,R}$, $C_{\calD_{d,R}}(\Theta)$ is polynomially bounded in
$(d,R)$ if and only if $H_{\calD_{d,R}}(\Theta)$ is.
\end{corollary}

\begin{proof}
For every $\mu\in\calD$, \cref{prop:conditional-rice} gives
$H_\mu(\Theta)=\lVert\rho_\mu\rVert_{L^\infty(\Theta)}$.
Taking the supremum over $\mu$ and applying \cref{eq:class-intensity}
proves the claim.
\end{proof}

\begin{proposition}[Joint coefficient-space area formula]
\label{prop:joint-rice}
If $\mu$ has Lebesgue density $f$ on $\R^d$, then $\Lambda_\mu$ is absolutely
continuous and its density satisfies, for almost every $x\in\Theta$,
\begin{align}
 \rho_\mu(x)
  =\int_{\R^{d-1}}
  |h_a'(x)|\,
  f\!\left(h_a(x),a\right)\,da.
 \label{eq:joint-rice}
\end{align}
The right-hand side is understood as an $L^1_{\mathrm{loc}}(\Theta)$
equivalence class and is independent of the chosen version of $f$.
\end{proposition}

\paragraph{Proof sketch.}
Apply the $d$-dimensional area formula to
$T(x,a)=(h_a(x),a)$.  Its geometric preimages are the distinct real roots and
$|\det DT|=|h_a'(x)|$; bounded-box exhaustion handles unbounded domains.

\begin{corollary}[Recovery of the polynomial joint-density bound]
\label{cor:joint-density-recovery}
Suppose $\mu$ is supported on $[-R,R]^d$ and has joint density at most
$\kappa$.  Then every bounded interval $I\subset\R$ satisfies
\begin{equation}
 \begin{aligned}
 &\Prob(\exists x\in I:P_\alpha(x)=0)\\
 &\quad\le \kappa(2R)^{d-1}
 \left(d+\frac{Rd(d-1)}2\right)|I|.
 \end{aligned}
 \label{eq:joint-density-recovery}
\end{equation}
Thus Proposition~\ref{prop:joint-rice} recovers the bound quoted as Theorem~2 in
\cite{balcan2026open}.
\end{corollary}

\paragraph{Proof sketch.}
On $|x|\le1$, \cref{eq:joint-rice} is bounded directly by the cube volume and
the derivative of the constant-coefficient root graph.  On $|x|>1$, reverse
the polynomial with $y=x^{-1}$, eliminate the next-to-leading coefficient,
and transform root intensity by $\rho_P(x)=|x|^{-2}\rho_Q(y)$.  The two chart
bounds give the same displayed constant.  The transformation identity is the
change-of-variables formula obtained by pushing the expected distinct-root
measure of $Q$ forward under $y\mapsto 1/y$ from $(-1,0)$ and $(0,1)$ onto the
two components of $\R\setminus[-1,1]$.

\subsection{Conditional-density bounds under dependent coefficients}
\label{sec:conditional}

Bounded marginal densities do not suffice.  Let $U$ be uniform on $[-1,1]$ and
take $d=2$,
\[
 \alpha_1=U,\qquad \alpha_0=-1-U.
\]
Both marginals have bounded densities and the coefficients lie in $[-2,2]$, but
$P_\alpha(1)=0$ almost surely.  Thus $c_\mu=\infty$.  The failure is conditional:
once $\alpha_1$ is known, $\alpha_0$ is a point mass.

Conditional anti-concentration of the constant coefficient removes exactly
that mechanism.

\begin{theorem}[Conditional-density bound]
\label{thm:conditional-bound}
Suppose $\alpha\in[-R,R]^d$ almost surely and a regular conditional law of
$\alpha_0$ given $A_-=a$ has a density satisfying
\begin{equation}
 \esssup_{u\in\R}q(u\mid a)\le K
 \quad\text{for $\nu$-almost every $a$}.
 \label{eq:conditional-essential-bound}
\end{equation}
Equivalently, $q(u\mid a)\le K$ for
$(du\otimes\nu(da))$-almost every $(u,a)$.  If $B>0$ and
$\Theta\subseteq[-B,B]$, then
\begin{equation}
 c_\mu(\Theta)
 \le K\left(dB^{d-1}+R\sum_{j=1}^{d-1}jB^{j-1}\right).
 \label{eq:conditional-polynomial}
\end{equation}
In particular, if $B\le1$,
\begin{equation}
 c_\mu(\Theta)\le K\left(d+\frac{Rd(d-1)}2\right).
 \label{eq:conditional-unit}
\end{equation}
\end{theorem}

\paragraph{Proof sketch.}
After conditioning on $A_-=a$, a root in $I$ is equivalent to
$\alpha_0\in h_a(I)$.  The image length is bounded by the total variation of
$h_a$ on $I$, and the conditional density converts that length into
probability.  Averaging gives \cref{eq:conditional-polynomial}.

For $d=1$, the root is $-\alpha_0$, so a $K$-bounded density gives the
global bound $K|I|$.  Only the eliminated coordinate must have a bounded
conditional density; the coordinates being conditioned on may be dependent or
singular.  This is a sufficient condition, while
\cref{thm:root-coordinates,cor:root-coordinate-class} give the
necessary-and-sufficient ordered-root criterion.

The local theorem uses only the constant-coefficient chart.  A second chart at
infinity removes the domain-radius dependence and gives a global result.

\begin{theorem}[Two-chart conditional bound]
\leavevmode\label{thm:two-chart}%
Suppose $d\ge2$ and $\alpha\in[-R,R]^d$ almost surely.  Let $\nu_0$ be the law
of $(\alpha_1,\ldots,\alpha_{d-1})$ and $\nu_{d-1}$ the law of
$(\alpha_0,\ldots,\alpha_{d-2})$.  Assume that regular conditional laws admit
jointly measurable densities satisfying
\begin{align*}
 q_0(u\mid a)&\le K_0
 &&(du\otimes\nu_0(da))\text{-a.e.},\\
 q_{d-1}(v\mid b)&\le K_{d-1}
 &&(dv\otimes\nu_{d-1}(db))\text{-a.e.}.
\end{align*}
Then every
bounded interval $I\subset\R$ satisfies
\begin{equation}
 \Prob(\exists x\in I:P_\alpha(x)=0)
 \le C_{\mathrm{2ch}}|I|,
 \label{eq:two-chart-bound}
\end{equation}
where
\begin{equation}
 C_{\mathrm{2ch}}
 =\max\!\left\{
 \begin{aligned}
 &K_0\!\left(d+\frac{Rd(d-1)}2\right),\\
 &K_{d-1}\!\left(1+\frac{Rd(d-1)}2\right)
 \end{aligned}
 \right\}.
 \label{eq:two-chart-constant}
\end{equation}
\end{theorem}

\paragraph{Proof sketch.}
On $|x|\le1$, eliminate $\alpha_0$ and use the derivative of $h_a$.  On
each exterior component, divide by $x^{d-1}$ and eliminate
$\alpha_{d-1}$; the resulting root graph has derivative at most
$1+Rd(d-1)/2$.  Splitting $I$ at $\pm1$ and summing proves the claim; the
split points are null under the conditional-density hypotheses.
\Cref{fig:two-chart} displays the two eliminations.

\begin{strip}
\noindent\begin{minipage}{\textwidth}
\vspace*{5pt}
\centering
\captionsetup{hypcap=false}
\begin{tikzpicture}[>=Latex,font=\footnotesize]
  \node[align=center,text width=4.6cm] at (-4.75,.78)
    {\textbf{outer chart} \quad $|x|\ge1$\\
     eliminate $\alpha_{d-1}$};
  \node[align=center,text width=4.2cm] at (0,.78)
    {\textbf{inner chart} \quad $|x|\le1$\\
     eliminate $\alpha_0=h_a(x)$};
  \node[align=center,text width=4.6cm] at (4.75,.78)
    {\textbf{outer chart} \quad $|x|\ge1$\\
     eliminate $\alpha_{d-1}$};

  \fill[gray!10] (-7,-.16) rectangle (-2.15,.16);
  \fill[blue!10] (-2.15,-.16) rectangle (2.15,.16);
  \fill[gray!10] (2.15,-.16) rectangle (7,.16);
  \draw[very thick,->] (-7,0)--(7.25,0) node[right] {$x$};
  \draw[thick] (-2.15,-.14)--(-2.15,.14)
    node[above=4pt] {$-1$};
  \draw[thick] (2.15,-.14)--(2.15,.14)
    node[above=4pt] {$1$};

  \draw[->,very thick,blue!65!black]
    (-5.35,-.27) to[out=-9,in=189] (-1.00,-.27);
  \draw[->,very thick,blue!65!black]
    (5.35,-.27) to[out=189,in=-9] (1.00,-.27);
  \node[blue!65!black,fill=white,inner sep=2pt] at (0,-.34)
    {$x\mapsto y=1/x$};
\end{tikzpicture}

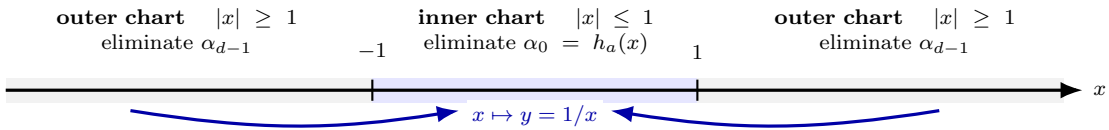
\captionof{figure}{Two-chart proof of \cref{thm:two-chart}.  The inner chart
eliminates $\alpha_0$; after $y=1/x$, the exterior charts eliminate
$\alpha_{d-1}$.}
\label{fig:two-chart}
\vspace*{5pt}
\end{minipage}
\end{strip}

\paragraph{A dependent-law calculation.}
Let $d\ge2$, fix $0<\eta\le R$, let $Y$ have any law, and let
$b_0,\ldots,b_{d-1}$ be measurable functions with
$|b_j(Y)|\le R$ for $1\le j\le d-2$ and
$|b_0(Y)|,|b_{d-1}(Y)|\le R-\eta$.  For independent
$V_0,V_{d-1}\sim\operatorname{Unif}[-\eta,\eta]$, also independent of $Y$,
set
\[
 \begin{aligned}
 \alpha_0&=b_0(Y)+V_0,\\
 \alpha_j&=b_j(Y), &&1\le j\le d-2,\\
\alpha_{d-1}&=b_{d-1}(Y)+V_{d-1}.
\end{aligned}
\]
Conditioning may change the posterior law of $Y$, but the fresh noise in the
eliminated coordinate remains independent.  Hence each conditional law is a
mixture of translates of the uniform density and is bounded by
$(2\eta)^{-1}$.  For example, a version of the first density is
\[
 q_0(u\mid a)=
 \E\!\left[
 \frac{\ind\{|u-b_0(Y)|\le\eta\}}{2\eta}
 \,\middle|\, A_-=a\right]\le\frac1{2\eta},
\]
and the opposite conditioning has the analogous formula.  Therefore
$K_0=K_{d-1}=(2\eta)^{-1}$, and
\begin{equation}
 c_\mu(\R)
 \le \frac{1}{2\eta}
 \left(d+\frac{Rd(d-1)}2\right).
 \label{eq:dependent-two-chart-example}
\end{equation}
The latent law of $Y$ may be discrete or singular.  If the rounds are
independent, all discontinuities are contained among the roots of at most $m$
such degree-$d$ transition equations per round, and the remaining hypotheses
of \cref{cor:dispersion-interface} hold, then $K\le md$ and that corollary
gives, for $\varepsilon\ge T^{-1/2}$, the uniform boundary count
$\widetilde O\!\left((m\eta^{-1}d(1+Rd)+1)T\varepsilon\right)$ and hence the
usual $\widetilde O(\sqrt T)$ regret.

\paragraph{Comparison with the prior polynomial bound.}
Theorem~2 in \cite{balcan2026open}, quoted from the supplementary theorem of
\cite{balcan2020semibandit}, gives for a joint density bounded by $\kappa$ the
bound
\begin{equation}
 \kappa(2R)^{d-1}
 \left(d+\frac{Rd(d-1)}2\right)|I|.
 \label{eq:prior-polynomial}
\end{equation}
Corollary~\ref{cor:joint-density-recovery} gives a direct root-intensity proof
of this bound.  The affine-image hypothesis in the prior result and the
conditional-density hypothesis in \cref{thm:conditional-bound} cover
complementary coefficient models.

\section{Affine recovery and the dispersion interface}
\label{sec:affine}

Fixing a leading coefficient turns a homogeneous random linear combination
into an affine family.  The incidence formula has an exact affine counterpart.

For $f_0\in C^1(\Theta)$ and a nonvanishing
$F\in C^1(\Theta;\R^N)$, write
\[
 H_x^\circ=\{a\in\operatorname{int}K:
 f_0(x)+\langle a,F(x)\rangle=0\}
\]
and define
\begin{equation}
 \begin{aligned}
 \lambda_K^{\rm aff}(f_0,F;x)
 &=\frac{1}{\operatorname{vol}_N(K)\lVert F(x)\rVert_2}\\
 &\quad\times\int_{H_x^\circ}|f_0'(x)+\langle a,F'(x)\rangle|\,
   d\mathcal H^{N-1}(a).
 \end{aligned}
 \label{eq:affine-incidence-speed}
\end{equation}
The integral is zero when $H_x^\circ$ is empty.  Let
$\Lambda_K^{\rm aff}$ be its essential supremum, and let
$C_{K,A}^{\rm aff}(f_0,F;\Theta)$ denote the worst-case interval
root-hitting constant over $\calD_{K,A}$ for
$f_0(x)+\langle a,F(x)\rangle$.

\begin{theorem}[Exact affine incidence]
\label{thm:exact-affine}
Under the preceding assumptions,
\begin{equation}
 C_{K,A}^{\rm aff}(f_0,F;\Theta)
 =A\Lambda_K^{\rm aff}(f_0,F;\Theta).
 \label{eq:exact-affine}
\end{equation}
\end{theorem}

The formula is invariant under multiplying $(f_0,F)$ by a nonvanishing
$C^1$ scalar function.  Its proof is the same incidence projection and local
sign-disagreement calculation as \cref{thm:body-characterization}, now for
the moving interior affine sections $H_x^\circ$; see
\cref{app:exact-affine-proof}.

\begin{lemma}[A cube-slab bound]
\label{lem:cube-slab}
Let $a$ have a density at most $\kappa$ on $Q_R$, let
$u\in\R^N$ satisfy $\lVert u\rVert_2=1$, and let $b\in\R$.  For every $h\ge0$,
\begin{equation}
 \Prob\bigl(|\langle a,u\rangle-b|\le h\bigr)
 \le A\frac{\sqrt 2}{R}h.
 \label{eq:cube-slab}
\end{equation}
\end{lemma}

\begin{corollary}[An affine cube-slab bound]
\label{thm:affine}
Let
\[
 \phi_a(x)=f_0(x)+\langle a,F(x)\rangle,
 \qquad a\in[-R,R]^N,
\]
where $F$ is nonvanishing.  Define
\[
 g(x)=\frac{F(x)}{\lVert F(x)\rVert_2},
 \qquad
 c(x)=-\frac{f_0(x)}{\lVert F(x)\rVert_2},
\]
and suppose $g$ and $c$ are respectively $L_g$- and $L_c$-Lipschitz on
$\Theta$.  Then every $\mu\in\calD^N_{R,\kappa}$ and
$I\in\calI(\Theta)$ satisfy
\begin{equation}
 \Prob_{a\sim\mu}(\exists x\in I:\phi_a(x)=0)
 \le A\sqrt2\left(\sqrt N\,L_g+\frac{L_c}{R}\right)|I|.
 \label{eq:affine-bound}
\end{equation}
\end{corollary}

\paragraph{Proof sketch.}
At a root, $\langle a,g(x_*)\rangle=c(x_*)$.  Lipschitz motion of both the
normal and offset places $a$ in a slab of half-width
$(R\sqrt N L_g+L_c)|I|$ around the equation at a fixed $x_0\in I$; the
cube-slab lemma gives \cref{eq:affine-bound}.

For the monic polynomial \cref{eq:monic-polynomial}, take
\begin{equation}
 f_0(x)=x^d,\qquad F(x)=(1,x,\ldots,x^{d-1}),\qquad N=d.
 \label{eq:monomial-feature}
\end{equation}
On the whole real line,
\begin{equation}
 \Lip\!\left(\frac{F}{\lVert F\rVert_2}\right)\le d-1,
 \qquad
 \Lip\!\left(\frac{x^d}{\lVert F\rVert_2}\right)\le2d-1.
 \label{eq:monomial-lipschitz}
\end{equation}
Separate estimates on $|x|\le1$ and $|x|\ge1$ give
\cref{eq:monomial-lipschitz}.  Substitution into \cref{eq:affine-bound} gives
the global estimate
\begin{equation}
 \begin{aligned}
 &\Prob(\exists x\in I:P_\alpha(x)=0)\\
 &\quad\le A\sqrt2\left(
      \sqrt d\,(d-1)+\frac{2d-1}{R}\right)|I|.
 \end{aligned}
 \label{eq:affine-monic-bound}
\end{equation}
The bound in \cref{eq:affine-monic-bound} is polynomial whenever $A$ is
polynomially controlled and $R$ is bounded below.
For joint-density laws, \cref{eq:prior-polynomial,eq:affine-monic-bound} have
different dependence on $d$ and $R$, so one may take their minimum;
\cref{thm:affine} additionally applies to arbitrary affine feature families.

\subsection{Dispersion interface}
\label{sec:dispersion}

Suppose round $t$ has at most $m$ transition equations and every such equation
obeys a common root-hitting bound $C|I|$.  For every fixed interval $I$,
let $H_I$ count the rounds for which some transition equation hits $I$.
Linearity of expectation and a union bound give
\begin{equation}
 \E H_I\le mCT|I|.
 \label{eq:expected-boundary-count}
\end{equation}

\begin{corollary}[One-dimensional dispersion interface]
\label{cor:dispersion-interface}
Let $u_1,\ldots,u_T$ be independent piecewise $L$-Lipschitz functions on a
compact interval $\Theta$ of length $W$.  Suppose each $u_t$ has at most $K$ discontinuities, all
contained among the roots of at most $m$ transition equations, and each
equation satisfies the common bound $C|I|$, where $K\ge1$.  If
 \[
 \begin{aligned}
 D(T,\varepsilon,\rho)=\#\{t:\;&u_t\text{ is not $L$-Lipschitz}\\
 &\text{on }\Theta\cap[\rho-\varepsilon,\rho+\varepsilon]\},
 \qquad \rho\in\Theta,
 \end{aligned}
 \]
then
\begin{equation}
 \begin{aligned}
 \E\max_\rho D(T,\varepsilon,\rho)
 &\le 2mCT\varepsilon\\
 &\quad+O\!\left(\sqrt{T\log(TK)}\right).
 \end{aligned}
 \label{eq:uniform-boundary-count}
\end{equation}
Consequently, for every $\varepsilon\ge T^{-1/2}$ this count is
$\widetilde O((mC+1)T\varepsilon)$.  Thus, when $m$, $C$, and $K$ have no dependence
on $T$ beyond suppressed polylogarithms, the sequence is $1/2$-dispersed.
Moreover, if $u_t:\Theta\to[0,1]$ and full-information feedback is available,
exponential weights on a grid of mesh at most $h$ satisfies
\begin{equation}
 \begin{aligned}
 \E\operatorname{Reg}_T
 &\le (L+2mC)Th
 +O\!\left(\sqrt{T\log(TK)}\right)\\
 &\quad+\sqrt{2T\log(W/h+2)}.
 \end{aligned}
 \label{eq:generic-grid-regret}
\end{equation}
Here $\operatorname{Reg}_T$ is regret to the best fixed parameter in
$\Theta$, with expectation over both the utility sequence and the forecaster.
In particular, $h=T^{-1/2}$ gives $\widetilde O((L+mC+1)\sqrt T)$ expected
regret whenever the displayed parameters have at most polylogarithmic
dependence on $T$.
\end{corollary}

\paragraph{Proof sketch.}
The fixed-interval expectation is at most $2mCT\varepsilon$ by a union bound.
The established uniformization theorem for independent piecewise-Lipschitz
functions contributes $O(\sqrt{T\log(TK)})$; when
$\varepsilon\ge T^{-1/2}$ this term is absorbed at the dispersion scale.  The
grid has at most $W/h+2$ experts.  Relative to a continuous comparator and
its nearest grid point, every round without a discontinuity in the intervening
$h$-neighborhood costs at most $Lh$; every remaining round costs at most one.
Apply \cref{eq:uniform-boundary-count} with $\varepsilon=h$ and add the
standard exponential-weights bound to obtain
\cref{eq:generic-grid-regret}.

In \cref{eq:expected-boundary-count,eq:uniform-boundary-count}, one may take
$C$ from \cref{cor:cube-comparison,thm:conditional-bound,thm:two-chart,thm:affine}
or from the density criteria in
\cref{thm:intensity,thm:root-coordinates}.

\section{Graph-learning applications}
\label{sec:applications}

\subsection{Cost-sensitive Gaussian-RBF graph SSL}
\label{sec:rbf-application}

Gaussian-RBF graphs are a standard family for graph-based semi-supervised
learning \cite{balcan2021ssl}.  Consider tuning their inverse bandwidth when
the downstream decision is cost-sensitive, a setting motivated by
class imbalance and unequal error costs \cite{wan2019cost}.  The analysis
combines the coefficient-body characterization with a boundary-count argument
to obtain dispersion and regret bounds.

Fix integers $n\ge2$ and $\ell\ge1$, a scale $D\ge1$, a parameter interval
$\Gamma=[0,B]$ with $B>0$, and $R,\kappa>0$.  On round $t$, an oblivious environment
fixes a vertex set $X_t=L_t\mathbin{\dot\cup}U_t$ with
$|X_t|\le n$, $|L_t|\ge\ell$, and both parts nonempty; binary labels
$y_t:X_t\to\{0,1\}$; and symmetric squared dissimilarities
\[
 d_t(a,b)\in[0,D],\qquad a\ne b.
\]
No discretization of the dissimilarities is imposed.  Parameter
$\gamma\in\Gamma$ defines the complete Gaussian-RBF graph
\begin{equation}
 w_{t,\gamma}(a,b)=\exp(-\gamma d_t(a,b)).
 \label{eq:rbf-kernel}
\end{equation}
Let $W_{t,\gamma}$ and $\mathsf D_{t,\gamma}$ be the adjacency and degree
matrices, and set
\[
 \begin{aligned}
 Q_{t,\gamma}&=\mathsf D_{UU,t,\gamma}-W_{UU,t,\gamma},
 &b_{t,\gamma}&=W_{UL,t,\gamma}y_{L,t},\\
 s_{U,t}(\gamma)&=Q_{t,\gamma}^{-1}b_{t,\gamma}.&&
 \end{aligned}
\]

Independently across rounds, a false-negative/false-positive cost pair
$(C_t^{\rm FN},C_t^{\rm FP})\in[0,R]^2$ is drawn from a joint density bounded
by $\kappa$, independently of the fixed base instances.  Given the partially
labeled graph and these costs, the configured algorithm predicts one at $v\in U_t$
exactly when
\begin{equation}
 C_t^{\rm FN}s_{t,v}(\gamma)
 \ge C_t^{\rm FP}\bigl(1-s_{t,v}(\gamma)\bigr),
 \label{eq:rbf-cost-sensitive-rule}
\end{equation}
with a fixed tie rule.  In the next display, write $\widehat y_{t,v}$ for the
prediction at $\gamma$.  After the labels on $U_t$ are revealed, its normalized
utility is
\begin{equation}
 \begin{aligned}
 u_t(\gamma)=1-\frac{1}{R|U_t|}\sum_{v\in U_t}\bigl[&
 C_t^{\rm FN}\ind\{y_t(v)=1,\widehat y_{t,v}=0\}\\
 {}+&C_t^{\rm FP}\ind\{y_t(v)=0,\widehat y_{t,v}=1\}
 \bigr].
 \end{aligned}
 \label{eq:rbf-utility}
\end{equation}
Thus $u_t(\gamma)\in[0,1]$.  The online protocol commits to $\gamma_t$ from
past feedback before round $t$; the current partially labeled graph and costs are then
used by the configured SSL algorithm, and full-information feedback reveals
the utility at every grid parameter after the held-out labels are observed.

Put
\begin{equation}
 A=(2R)^2\kappa,\qquad \omega=e^{-BD},\qquad
 K_n=2n\,n!\,n^n.
 \label{eq:rbf-scales}
\end{equation}
In particular, $\log K_n=O(n\log n)$.

\begin{theorem}[Cost-sensitive Gaussian-RBF graph SSL]
\label{thm:rbf-application}
Under the preceding assumptions, for every $v\in U_t$ and every interval
$I\subseteq\Gamma$,
\begin{equation}
 \begin{aligned}
 &\Prob\!\left(\exists\gamma\in I:
 C_t^{\rm FN}s_{t,v}(\gamma)
 =C_t^{\rm FP}(1-s_{t,v}(\gamma))\right)\\
 &\hspace{25mm}\le \frac{A(n-1)D}{\ell\omega}|I|.
 \end{aligned}
 \label{eq:rbf-root-hitting}
\end{equation}
The utility $u_t$ is piecewise constant with at most
$K_n$ discontinuities.  Consequently,
\begin{equation}
 \begin{aligned}
 \E\max_\rho D(T,\varepsilon,\rho)
 &\le \frac{2An(n-1)D}{\ell\omega}T\varepsilon\\
 &\quad+O\!\left(\sqrt{T\log(TK_n)}\right).
 \end{aligned}
 \label{eq:rbf-dispersion}
\end{equation}
In particular, the utilities are $1/2$-dispersed.  Under full-information
feedback, exponential weights on a grid of spacing $T^{-1/2}$ satisfies
\begin{equation}
 \begin{aligned}
 &\E\!\left[
 \max_{\gamma\in\Gamma}\sum_{t=1}^T u_t(\gamma)
  -\sum_{t=1}^T u_t(\gamma_t)\right]\\
 &\quad\le \frac{2An(n-1)D}{\ell\omega}\sqrt T
 +O\!\left(\sqrt{T\log(TK_n)}\right)\\
 &\qquad+O\!\left(\sqrt{T\log(B\sqrt T+2)}\right).
 \end{aligned}
 \label{eq:rbf-regret}
\end{equation}
Thus the expected regret is
$\widetilde O((An^2D e^{BD}/\ell+1)\sqrt T)$, where
$e^{BD}=\omega^{-1}$ records the conditioning of the graph family.
\end{theorem}

\paragraph{Proof sketch.}
Every RBF edge weight lies in $[\omega,1]$.  The grounded-Laplacian energy and
maximum principle give
\[
 Q_{t,\gamma}\succ0,\qquad
 \lVert Q_{t,\gamma}^{-1}\rVert_\infty\le(\ell\omega)^{-1},
 \qquad 0\le s_{t,v}(\gamma)\le1.
\]
Differentiating the harmonic equations then yields
\begin{equation}
 \lVert s'_{U,t}(\gamma)\rVert_\infty
 \le\frac{(n-1)D}{\ell\omega}.
 \label{eq:rbf-score-speed}
\end{equation}
For the smooth, nonvanishing two-dimensional feature curve
\[
 F_{t,v}(\gamma)=(s_{t,v}(\gamma),1-s_{t,v}(\gamma)),
\]
a direct section calculation gives
\begin{equation}
 \begin{aligned}
 \lambda_{Q_R}(F_{t,v};\gamma)
 &=\frac{|s'_{t,v}(\gamma)|}
 {4\max\{s_{t,v}(\gamma),1-s_{t,v}(\gamma)\}^2}\\
 &\le |s'_{t,v}(\gamma)|.
 \end{aligned}
 \label{eq:rbf-exact-incidence-speed}
\end{equation}
The reflected cost vector
$(C_t^{\rm FN},-C_t^{\rm FP})$ has density at most $\kappa$ in $Q_R$.
Applying \cref{thm:body-characterization} to
\cref{eq:rbf-exact-incidence-speed} proves
\cref{eq:rbf-root-hitting}.

For boundary complexity, let $r_t=|U_t|$.  Every entry of
$Q_{t,\gamma}$ and $b_{t,\gamma}$ is a sum of at most $n$ exponentials
$e^{-\beta\gamma}$ with $\beta\in[0,D]$.  Determinant expansion shows that
$\Delta_t(\gamma)=\det Q_{t,\gamma}$ and
$N_{t,v}(\gamma)=e_v^\top\operatorname{adj}(Q_{t,\gamma})b_{t,\gamma}$
each contain at most $r_t!n^{r_t}$ exponential terms before like exponents
are collected, while $\Delta_t(\gamma)>0$.  The prediction at $v$ can
therefore change only at a root of
\begin{equation}
 H_{t,v}(\gamma)
 =(C_t^{\rm FN}+C_t^{\rm FP})N_{t,v}(\gamma)
 -C_t^{\rm FP}\Delta_t(\gamma).
 \label{eq:rbf-transition-exponential-polynomial}
\end{equation}
After collection, \cref{lem:exp-poly-zeros} shows that a nonzero exponential
polynomial with $M$ distinct real exponents has at most $M-1$ distinct real
zeros.  Thus
\cref{eq:rbf-transition-exponential-polynomial} has fewer than
$2r_t!n^{r_t}$ roots unless it vanishes identically; in that case the fixed
tie rule makes the prediction constant.  Summing over vertices gives the
uniform bound $K_n$, and applying \cref{cor:dispersion-interface} gives
\cref{eq:rbf-dispersion}.  Grid approximation and the finite-expert
exponential-weights bound give \cref{eq:rbf-regret}.

Projective motion of $F_{t,v}$ bounds the probability of a root.  Expanding the
transition equation into exponentials bounds the number of roots.  The
expanded coefficient vector is a two-dimensional linear image of the cost pair
and is generally singular in higher ambient dimension.

\subsection{Common-offset polynomial-kernel graph SSL}
\label{sec:ssl-application}

Consider the polynomial graph kernel of \citet{balcan2021ssl} when every
similarity in an instance shares a calibration offset.  Such offsets arise when
a sensor or preprocessing pipeline adds the same bias to all pairwise scores.

Fix a kernel degree $q\ge1$, a graph-size bound $n\ge2$, and a parameter
interval $\Lambda=[0,B]$ with $B>0$.  On round $t$, an oblivious environment
specifies a vertex set $X_t=L_t\mathbin{\dot\cup}U_t$, where both parts are
nonempty and $|X_t|\le n$, binary labels $y_t:X_t\to\{0,1\}$, and symmetric
latent similarities
\[
 s_t(a,b)\in[s_-,s_+] \qquad(a\ne b),
\]
where $0<\eta<s_-\le s_+<\infty$.  Independently across rounds, a scalar
calibration offset $V_t$ is drawn from $[-\eta,\eta]$ with a Lebesgue density
bounded by $\kappa$.  The observed similarities are
\[
 \widetilde s_t(a,b)=s_t(a,b)+V_t,
\]
and parameter $\lambda\in\Lambda$ produces the complete weighted graph
\begin{equation}
 w_{t,\lambda}(a,b)
 =\bigl(\widetilde s_t(a,b)+\lambda\bigr)^q.
 \label{eq:ssl-kernel}
\end{equation}
All weights are positive.  Let $W_{t,\lambda}$ and $\mathsf D_{t,\lambda}$
be its adjacency and degree matrices, and put
\begin{equation}
 \begin{aligned}
 Q_{t,\lambda}
 &=\mathsf D_{UU,t,\lambda}-W_{UU,t,\lambda},\\
 b_{t,\lambda}&=W_{UL,t,\lambda}y_{L,t},\\
 f_{U,t}(\lambda)&=Q_{t,\lambda}^{-1}b_{t,\lambda}.
 \end{aligned}
 \label{eq:ssl-harmonic-solution}
\end{equation}
This is the harmonic extension minimizing the weighted quadratic energy with
the labels on $L_t$ fixed.  Vertex $a\in U_t$ is assigned label one when
$f_{U,t}(\lambda)_a\ge1/2$, with this fixed rule also resolving ties.  The
utility $u_t(\lambda)\in[0,1]$ is the fraction of correctly classified
vertices in $U_t$.

The online forecaster chooses $\lambda_t$ using only past feedback.  The
current observed graph is then passed to the configured classifier; after the
labels on $U_t$ are revealed, full-information feedback provides the utility
at every grid parameter.  In particular, the forecaster does not observe the
current offset before committing to $\lambda_t$.

For the transition equations, write $r_t=|U_t|$ and use
$\gamma=\lambda+V_t$.  Define $Q_t(\gamma)$ and $b_t(\gamma)$ by replacing
each edge weight in \cref{eq:ssl-harmonic-solution} by
$(s_t(a,b)+\gamma)^q$, and for $a\in U_t$ set
\begin{equation}
 \psi_{t,a}(\gamma)
 =2e_a^\top\operatorname{adj}(Q_t(\gamma))b_t(\gamma)
  -\det Q_t(\gamma).
 \label{eq:ssl-transition-polynomial}
\end{equation}

\begin{theorem}[Common-offset polynomial-kernel SSL]
\label{thm:ssl-application}
Under the preceding assumptions, every nonzero polynomial
$\psi_{t,a}$ has degree at most $qr_t$.  For each such polynomial and every
interval $I\subseteq\Lambda$,
\begin{equation}
 \Prob\bigl(\exists\lambda\in I:
       \psi_{t,a}(\lambda+V_t)=0\bigr)
 \le qr_t\kappa|I|.
 \label{eq:ssl-root-hitting}
\end{equation}
The utility $u_t$ is piecewise constant with at most
$qr_t^2\le qn^2$ discontinuities, all contained among the roots of at most
$r_t$ nonzero transition polynomials from
\cref{eq:ssl-transition-polynomial}.  Hence
\begin{equation}
 \begin{aligned}
 \E\max_\rho D(T,\varepsilon,\rho)
 &\le 2qn^2\kappa T\varepsilon\\
 &\quad+O\!\left(\sqrt{T\log(Tqn^2)}\right).
 \end{aligned}
 \label{eq:ssl-dispersion}
\end{equation}
In particular, the utilities are $1/2$-dispersed.

Under full-information feedback, exponential weights on a grid of spacing
$T^{-1/2}$ produces parameters $\lambda_1,\ldots,\lambda_T$ satisfying
\begin{equation}
 \begin{aligned}
 &\E\!\left[
   \max_{\lambda\in\Lambda}\sum_{t=1}^T u_t(\lambda)
   -\sum_{t=1}^T u_t(\lambda_t)
 \right]\\
 &\quad\le 2qn^2\kappa\sqrt T\\
 &\quad+O\!\left(\sqrt{T\log(Tqn^2)}\right)\\
 &\quad+O\!\left(\sqrt{T\log(B\sqrt T+2)}\right).
 \end{aligned}
 \label{eq:ssl-regret}
\end{equation}
Thus the expected regret is
$\widetilde O((qn^2\kappa+1)\sqrt T)$, with polynomial dependence on all
displayed problem parameters.
\end{theorem}

\paragraph{Proof sketch.}
Positivity of the complete graph gives $Q_t(\gamma)\succ0$ throughout the
relevant range of $\gamma$.  Consequently,
\[
 2f_{U,t}(\lambda)_a-1
 =\frac{\psi_{t,a}(\lambda+V_t)}
        {\det Q_t(\lambda+V_t)}.
\]
Every entry of $Q_t(\gamma)$ and $b_t(\gamma)$ has degree at most $q$,
so \cref{eq:ssl-transition-polynomial} has degree at most $qr_t$.
If it vanishes identically, the fixed tie rule makes the corresponding
prediction constant.  Otherwise, let its distinct real roots be
$z_1<\cdots<z_M$, where $M\le qr_t$.  The roots in the online parameter are
exactly
\[
 z_1-V_t,\ldots,z_M-V_t.
\]
Viewed on the whole real line, the $j$th ordered root is $z_j-V_t$; its law is
a translate of the reflected offset law and has density at most $\kappa$.
Restricting the hitting event to $I\subseteq[0,B]$ and taking a union bound
proves
\cref{eq:ssl-root-hitting}; summing over at most $r_t$ vertices gives the
root-hitting coefficient $qn^2\kappa$ for the round.  The fixed latent
instances and independent offsets make the utilities independent, so
\cref{cor:dispersion-interface} gives \cref{eq:ssl-dispersion}.  Finally, the
nearest grid point to the best continuous parameter differs from it only on
rounds having a discontinuity within distance $T^{-1/2}$.  Combining
\cref{eq:ssl-dispersion} with the finite-expert exponential-weights bound gives
\cref{eq:ssl-regret}.

\paragraph{Singularity of the induced law.}
Because the same scalar offset is added to every edge, the observed similarity
vector lies on the affine line
$(s_t(a,b))_{\{a,b\}}+V_t\mathbf 1$; for a multi-edge graph, its law is
singular in similarity space.  The same one-dimensional structure appears in
coefficient space.  After expansion and monic normalization, the nonleading
coefficients of a realized transition polynomial of degree $m\ge2$ form a
polynomial image of $V_t$ in $\R^m$, so their law is singular.  Its ordered
roots remain $z_j-V_t$, and each has density at most $\kappa$.

\section{Conclusion}
\label{sec:conclusion}

The coefficient-body formula weights projective velocity by central sections
of the support body.  For cube-supported bounded-density laws, it is equivalent
up to universal constants to $A L_{\mathrm P}$ and removes the previous
$\sqrt N$ factor.  For general polynomial coefficient laws, bounded densities
of the ordered-root pushforwards characterize finite and polynomial root
constants up to the factor $d$.  The conditional and joint area formulas
express this condition in coefficient coordinates.

For cost-sensitive Gaussian-RBF graph SSL, projective incidence controls each
harmonic threshold and yields
$\widetilde O((An^2D e^{BD}/\ell+1)\sqrt T)$ expected regret.  For the
common-offset polynomial kernel, rigid ordered-root translation yields
$\widetilde O((qn^2\kappa+1)\sqrt T)$ regret even when the induced law is
singular in the ambient coefficient space.

\bibliographystyle{plainnat}
\bibliography{references}

\clearpage
\appendix
\onecolumn
\thispagestyle{plain}

\begin{center}
{\LARGE\bfseries Appendix\par}
\vspace{0.5em}
{\large Sharp Root Anti-Concentration via Projective Incidence and Ordered Root Laws\par}
\end{center}
\vspace{0.5em}
\hrule
\vspace{1.25em}

\section{Proof of the coefficient-body characterization}
\label{app:projective-proofs}

\subsection{Incidence area and first variation}
\label{app:incidence-area}

\begin{proof}[Proof of \cref{lem:incidence-area}]
Endpoints of $I$ and $\partial K$ affect neither side.  On
$\operatorname{int}K\times\operatorname{int}I$, consider
\[
 M_I=\{(a,x):q(a,x)=0\},
 \qquad q(a,x)=\langle a,F(x)\rangle.
\]
This is a regular $N$-dimensional hypersurface because
\[
 \nabla q(a,x)=(F(x),\langle a,F'(x)\rangle)
\]
and $F(x)\ne0$.  Let $\pi(a,x)=a$ and $s(a,x)=x$.  A unit normal to
$M_I$ is
\[
 \nu(a,x)=
 \frac{(F(x),\langle a,F'(x)\rangle)}
 {\sqrt{\lVert F(x)\rVert_2^2+\langle a,F'(x)\rangle^2}}.
\]
The $N$-Jacobian of the projection $\pi|_{M_I}$ is the absolute last
coordinate of this normal:
\[
 J_N^{M_I}\pi
 =\frac{|\langle a,F'(x)\rangle|}
 {\sqrt{\lVert F(x)\rVert_2^2+\langle a,F'(x)\rangle^2}}.
\]
Likewise,
\[
 J_1^{M_I}s=\lVert\nabla_{M_I}s\rVert_2
 =\frac{\lVert F(x)\rVert_2}
 {\sqrt{\lVert F(x)\rVert_2^2+\langle a,F'(x)\rangle^2}}.
\]
The area and coarea formulas used below are in
\citet[Section~3.2]{federer1969geometric}.  The area formula for
$\pi|_{M_I}$ counts the number of incidences above
almost every $a$, whereas $E_F(I)$ records only whether an incidence exists.
The coarea formula for $s|_{M_I}$ therefore gives
\begin{align}
 \operatorname{vol}_N(E_F(I))
 &\le\int_{M_I}J_N^{M_I}\pi\,d\mathcal H^N\notag\\
 &=\int_I\int_{K\cap F(x)^\perp}
   \frac{|\langle a,F'(x)\rangle|}{\lVert F(x)\rVert_2}
   \,d\mathcal H^{N-1}(a)\,dx.
 \label{eq:incidence-area-proof}
\end{align}
This is \cref{eq:incidence-area-bound}.

It remains to prove the first variation.  Fix an interior $x$, put
$u=F(x)$, and write
\[
 u_h=F(x+h)=u+hw_h,\qquad w_h\longrightarrow w:=F'(x).
\]
Set
\[
 D_h=\{a\in K:\langle a,u\rangle\langle a,u_h\rangle<0\},
 \qquad H=u^\perp.
\]
Apply coarea to $a\mapsto\langle a,u\rangle$, substitute $t=hs$, and
parameterize the level set $\langle a,u\rangle=hs$ by
$a=z+hsu/\lVert u\rVert_2^2$ with $z\in H$.  This yields
\begin{align}
 \frac{\operatorname{vol}_N(D_h)}{|h|}
 =\frac1{\lVert u\rVert_2}\int_{\R}\int_H
 &\ind_K\!\left(z+\frac{hs}{\lVert u\rVert_2^2}u\right)\notag\\[-1mm]
 {}\times&
 \ind\!\left\{
 s\left(s+\left\langle
 z+\frac{hs}{\lVert u\rVert_2^2}u,w_h
 \right\rangle\right)<0\right\}
 d\mathcal H^{N-1}(z)\,ds.
 \label{eq:first-variation-coarea}
\end{align}
Because $K$ and $(w_h)$ are bounded, the second indicator can be nonzero
only for $|s|$ below a constant independent of small $h$; the first confines
$z$ to a fixed bounded set.  The relative boundary of the
$(N-1)$-dimensional convex body $K\cap H$ has
$\mathcal H^{N-1}$-measure zero.  Dominated convergence in
\cref{eq:first-variation-coarea} now gives
\begin{align*}
 \lim_{h\to0}\frac{\operatorname{vol}_N(D_h)}{|h|}
 &=\frac1{\lVert u\rVert_2}
 \int_{K\cap H}\int_\R
 \ind\{s(s+\langle z,w\rangle)<0\}\,ds\,
 d\mathcal H^{N-1}(z)\\
 &=\frac1{\lVert u\rVert_2}
 \int_{K\cap u^\perp}|\langle z,w\rangle|\,
 d\mathcal H^{N-1}(z),
\end{align*}
because the inner interval has length $|\langle z,w\rangle|$.
Division by $\operatorname{vol}_N(K)$ proves
\cref{eq:incidence-first-variation}.
\end{proof}

\subsection{Exact worst-case density cap}
\label{app:body-characterization}

\begin{proof}[Proof of \cref{thm:body-characterization}]
For $\mu\in\calD_{K,A}$ and $I\in\calI(\Theta)$,
\begin{align*}
 \Prob_{\mu}(\exists x\in I:\langle a,F(x)\rangle=0)
 &\le\frac{A}{\operatorname{vol}_N(K)}
       \operatorname{vol}_N(E_F(I))\\
 &\le A\int_I\lambda_K(F;x)\,dx\\
 &\le A\Lambda_K(F;\Theta)|I|.
\end{align*}
This proves $C_{K,A}\le A\Lambda_K$.

For the reverse inequality, choose a point $x\in\Theta^\circ$ at which
\cref{eq:incidence-first-variation} holds.  Let
\[
 D_h=\{a\in K:
   \langle a,F(x)\rangle\langle a,F(x+h)\rangle<0\},
 \qquad
 q_h=\frac{\operatorname{vol}_N(D_h)}{\operatorname{vol}_N(K)}.
\]
For small enough $h>0$, $q_h\le1/A$.  Define a density relative to the
uniform law on $K$ by
\begin{equation}
 r_h(a)=A\ind_{D_h}(a)
       +\frac{1-Aq_h}{1-q_h}\ind_{K\setminus D_h}(a).
 \label{eq:capped-extremal-density}
\end{equation}
The second coefficient lies in $[0,1]$, so $0\le r_h\le A$, and
$\int r_h\,d\operatorname{Unif}(K)=1$.  The induced law belongs to
$\calD_{K,A}$.  Every $a\in D_h$, apart from two null endpoint hyperplanes,
has a root in $[x,x+h]$ by continuity.  Consequently,
\[
 C_{K,A}(F;\Theta)
 \ge\lim_{h\downarrow0}\frac{Aq_h}{h}
 =A\lambda_K(F;x).
\]
Taking points whose incidence speeds approach the essential supremum proves
the lower bound.  If the essential supremum is infinite, the same argument
gives arbitrarily large hitting ratios.

Because $C_{K,A}$ is the supremum jointly over admissible laws and intervals,
the law in \cref{eq:capped-extremal-density} may depend on $x$ and $h$.
The equality need not hold under additional restrictions such as
log-concavity, isotropy, or symmetry.

If $F(x_*)=0$, every coefficient vector has a root at $x_*$.  Intervals
containing $x_*$ can have arbitrarily small positive length, so the hitting
constant is infinite.
\end{proof}

\subsection{Dimension-free comparison for the cube}
\label{app:projective-characterization-proof}

\begin{proof}[Proof of \cref{cor:cube-comparison}]
The case $N=1$ is immediate: a nonzero homogeneous scalar equation has no
root, and real projective space has one point.  Suppose $N\ge2$, put
$G=F/\lVert F\rVert_2$, and fix an $x$ with $G'(x)\ne0$.  Let
$e=G'(x)/\lVert G'(x)\rVert_2$.  If $X$ is uniform on the isotropic cube
$[-\sqrt3,\sqrt3]^N$, the pair
\[
 (\langle X,G(x)\rangle,\langle X,e\rangle)
\]
has an isotropic log-concave density on $\R^2$.  Let $q_x$ denote its
upper-semicontinuous log-concave representative, equivalently the fiber-density
representative supplied by coarea on the interior of its support.  The coarea formula
and scalar invariance under the rescaling from $Q_R$ give
\begin{equation}
 \lambda_{Q_R}(F;x)
 =\lVert G'(x)\rVert_2
   \int_{\R}|y|q_x(0,y)\,dy.
 \label{eq:cube-speed-marginal}
\end{equation}

There are universal constants $a_0,b_0,a_1,b_1>0$ such that every isotropic
log-concave density $q$ on $\R^2$ satisfies
\[
 q(0,y)\ge a_0\quad (|y|\le b_0),
 \qquad
 q(0,y)\le a_1e^{-b_1|y|}\quad (y\in\R).
\]
These fixed-dimensional bounds follow, for example, from
\citet[Lemma~2(b,c)]{balcan2013linear}; see also
\citet{lovasz2007geometry}.  Hence
\begin{equation}
 a_0b_0^2\lVert G'(x)\rVert_2
 \le\lambda_{Q_R}(F;x)
 \le\frac{2a_1}{b_1^2}\lVert G'(x)\rVert_2.
 \label{eq:cube-speed-comparison}
\end{equation}
The same inequalities hold when $G'(x)=0$.

The quotient map from the unit sphere to real projective space is a local
isometry.  Therefore the metric derivative of $x\mapsto[F(x)]$ is
$\lVert G'(x)\rVert_2$, and for a $C^1$ curve on an interval,
\[
 L_{\mathrm P}(F;\Theta)
 =\operatorname*{ess\,sup}_{x\in\Theta^\circ}
   \lVert G'(x)\rVert_2.
\]
Take essential suprema in \cref{eq:cube-speed-comparison} and apply
\cref{thm:body-characterization} with $K=Q_R$.  This proves
\cref{eq:projective-characterization} with
$c_{\mathrm{cube}}=a_0b_0^2$ and
$C_{\mathrm{cube}}=2a_1/b_1^2$.
\end{proof}

\subsection{Compact and scaling examples}
\label{app:compact-bound-proof}

\begin{proof}[Proof of \cref{cor:compact-dichotomy}]
Compactness gives $m=\inf_J\lVert F\rVert_2>0$.  For
$G=F/\lVert F\rVert_2$,
\[
 G'(x)=\frac{(\mathrm{Id}-G(x)G(x)^\top)F'(x)}
             {\lVert F(x)\rVert_2},
\]
so $L_{\mathrm{P}}(F;J)\le\sup_J\lVert F'\rVert_2/m$.
Apply \cref{cor:cube-comparison}.
\end{proof}

For the calculation in \cref{eq:scaled-exact-sharpness}, the active
two-dimensional coordinates of $F_{\delta,N}(x)$ are $(1,x/\delta)$.
Writing $s=x/\delta$, the section of $Q_R$ perpendicular to this vector can
be parameterized in the first two coordinates by $(-sy,y)$, where
$|y|\le R/\max\{1,|s|\}$.  The factor
$\sqrt{1+s^2}$ in the section measure cancels
$\lVert F_{\delta,N}(x)\rVert_2$.  The remaining coordinates contribute
$(2R)^{N-2}$, and hence
\[
 \lambda_{Q_R}(F_{\delta,N};x)
 =\frac{1}{4\delta\max\{1,(x/\delta)^2\}}.
\]
Its maximum on $[-1,1]$ is $1/(4\delta)$.

\subsection{The cube-slab estimate used for affine families}
\label{app:cube-slab-proof}

\begin{proof}[Proof of \cref{lem:cube-slab}]
For $t\in\R$, let
\[
 S_t=Q_R\cap\{x:\langle x,u\rangle=t\}.
\]
Brunn's concavity theorem and the central symmetry of $Q_R$ imply that
$\operatorname{vol}_{N-1}(S_t)$ is maximized at $t=0$.  Ball's cube-slicing
theorem \cite{ball1986cube}, rescaled from the unit cube to $Q_R$, gives
\[
 \operatorname{vol}_{N-1}(S_t)
 \le\sqrt2\,(2R)^{N-1}.
\]
The coarea formula bounds the slab volume by
$2h\sqrt2\,(2R)^{N-1}$.  Multiplication by
$\kappa=A/(2R)^N$ proves \cref{eq:cube-slab}; the case $N=1$ is immediate.
\end{proof}

\section{Proofs of the arbitrary-law polynomial criteria}
\label{app:arbitrary-law}

\subsection{Borel root selectors and the expected root measure}
\label{app:root-selectors}

\begin{lemma}[Borel root selectors]
\label{lem:root-kernel-measurable}
There are Borel maps $r_1,\ldots,r_d:\R^d\to\mathbb C$ that enumerate the
roots of $P_\alpha$ with algebraic multiplicity.  For every Borel
$B\subseteq\Theta$, define
\[
 \xi_\alpha(B)=
 \sum_{j=1}^d
 \ind\!\left\{r_j(\alpha)\in B\cap\R,
        \ r_j(\alpha)\ne r_i(\alpha)\ \text{for all }i<j\right\}
 .
\]
Then $\xi$ is a measurable finite point-measure kernel and
$Z_\alpha(B)=\xi_\alpha(B)$ is measurable in $\alpha$ for every Borel $B$.
\end{lemma}

\begin{proof}
The unordered multiset of roots of a monic polynomial depends continuously on
its coefficient vector in the quotient topology on $\mathbb C^d$ modulo
permutations; this also follows by viewing the roots as the eigenvalue multiset
of the companion matrix.  Sort each multiset by the lexicographic Borel order on
$(\operatorname{Re}z,\operatorname{Im}z)$, retaining repetitions.  Sorting a
finite multiset is Borel, which gives the maps $r_1,\ldots,r_d$.  Each summand in
the display is Borel in $\alpha$ for every Borel $B$, and the finite sum is a
counting-measure kernel.  The condition excluding
earlier equal roots removes algebraic duplicates without discarding
discriminant-zero vectors.
\end{proof}

It follows that $\Lambda_\mu=\E\xi_\alpha$ is a well-defined finite Borel
measure with total mass at most $d$, for every coefficient law, including laws
supported on the discriminant.  Filtering the finite Borel enumeration to
distinct roots in $\R\cap\Theta$ and sorting them increasingly also shows that
$m(\alpha)$ and every partial map $r_k$ used in
\cref{eq:ordered-root-law} are Borel.  Equivalently, one may assign a cemetery
value in an adjoined abstract point $\partial\notin\R$ to unused positions and
obtain $d$ total Borel maps into $\Theta\cup\{\partial\}$.

The extension from interval domination to Borel domination used in
\cref{thm:intensity} is also valid when $\Theta$ is unbounded.  Exhaust every
component of an open set from within by closed intervals in
$\calI(\Theta)$.  Continuity from below first gives the bound on each component;
countable additivity gives it on the open set.  Since
$\Lambda_\mu(\Theta)\le d$, outer regularity then yields the inequality for all
Borel sets.

\subsection{Root-intensity characterization}
\label{app:intensity-proof}

\begin{proof}[Proof of \cref{thm:intensity}]
For every interval $I$ and every coefficient vector,
\begin{equation}
 \ind\{Z_\alpha(I)\ge1\}
 \le Z_\alpha(I)
 \le d\,\ind\{Z_\alpha(I)\ge1\}.
 \label{eq:root-count-sandwich}
\end{equation}
Taking expectations gives
\begin{equation}
 \Prob(Z_\alpha(I)\ge1)
 \le\Lambda_\mu(I)
 \le d\Prob(Z_\alpha(I)\ge1).
 \label{eq:expected-sandwich}
\end{equation}
If $\Lambda_\mu$ has density bounded by $K$, the first inequality gives
$c_\mu\le K$.

Conversely, if $c_\mu=C<\infty$, the second inequality yields
$\Lambda_\mu(I)\le dC|I|$ for every closed $I\in\calI(\Theta)$.  Each component
of an open subset of $\Theta$ is exhausted from within by increasing closed
intervals, with one-sided exhaustion at a boundary of $\Theta$.  Continuity
from below and countable additivity therefore give the same domination for open
sets.  Outer regularity then gives
$\Lambda_\mu(B)\le dC\Leb(B)$ for every Borel set $B$.  The Radon--Nikodym
theorem supplies a density bounded almost everywhere by $dC$.
\end{proof}

\subsection{Ordered-root pushforwards}
\label{app:root-coordinates-proof}

\begin{proof}[Proof of \cref{thm:root-coordinates}]
For every $k$ and every $I\in\calI(\Theta)$,
\[
 \nu_{\mu,k}(I)
 \le\Prob\bigl(Z_\alpha(I)\ge1\bigr)
 \le c_\mu(\Theta)|I|.
\]
The same interval-to-Borel domination used for $\Lambda_\mu$ shows that each
$\nu_{\mu,k}$ has density bounded by
$c_\mu$ whenever $c_\mu<\infty$.  This proves the left inequality, including
its contrapositive when a bounded density fails to exist.

Conversely, the hitting event is the union, over $1\le k\le d$, of
$\{m(\alpha)\ge k,r_k(\alpha)\in I\}$.  Hence
\[
 \Prob(Z_\alpha(I)\ge1)
 \le\sum_{k=1}^d\nu_{\mu,k}(I)
 \le dM_\mu(\Theta)|I|,
\]
which proves the right inequality.
\end{proof}

The class statement in \cref{cor:root-coordinate-class} follows by taking the
supremum over $\mu\in\calD$ in the two-sided inequality.  Since the only loss
is $d$, the same argument proves the polynomial-growth equivalence for classes
indexed by $(d,R)$.

\section{Coefficient-space formulas and density conditions}
\label{app:coefficient-formulas}

\subsection{Conditional formula}
\label{app:conditional-formula}

The density $q$ may be redefined as zero on the product-null set where a
version takes the value $+\infty$.  This leaves the conditional laws unchanged
and avoids a $0\cdot\infty$ convention at critical points.

Fix $a$ for which $q(\cdot\mid a)$ exists.  On a compact interval $K$, the
polynomial $h_a$ is Lipschitz.  The one-dimensional area formula states that for
every nonnegative Borel function $r$,
\begin{equation}
 \int_K r(x)|h_a'(x)|\,dx
 =\int_\R\sum_{x\in K:h_a(x)=u}r(x)\,du.
 \label{eq:one-dimensional-area}
\end{equation}
Apply it with $r(x)=\ind_B(x)q(h_a(x)\mid a)$.  The right-hand side is
\[
 \int_\R q(u\mid a)
 \#\{x\in B\cap K:h_a(x)=u\}\,du.
\]
The count is the number of distinct roots of
$x^d+\sum_{j=1}^{d-1}a_jx^j+u$.  Increasing compact intervals exhaust
$\Theta$; monotone convergence then gives
\begin{equation}
 \int q(u\mid a)Z_{(u,a)}(B)\,du
 =\int_B |h_a'(x)|q(h_a(x)\mid a)\,dx.
 \label{eq:conditional-area-complete}
\end{equation}
The area formula counts geometric preimages rather than algebraic
multiplicity, which matches the definition of $Z$.  Its equality is an
integral identity; the finitely many critical values of $h_a$ form a Lebesgue
null set and therefore contribute no additional term under the conditional
density.

Joint measurability of $q$ and Tonelli's theorem allow integration of
\cref{eq:conditional-area-complete} against $\nu(da)$.  Disintegration gives
the left side as $\Lambda_\mu(B)$; Tonelli gives the right side as the integral
over $B$ of \cref{eq:conditional-rice}.  This proves
Proposition~\ref{prop:conditional-rice}.

Let $q$ and $\widetilde q$ be two jointly measurable
versions of the conditional density.  They agree for
$(du\otimes\nu(da))$-almost every $(u,a)$.  Applying
\cref{eq:one-dimensional-area} to their absolute difference and then Tonelli
gives, on every compact $K\subset\Theta$,
\[
 \int\!\int_K |h_a'(x)|
 \bigl|q(h_a(x)\mid a)-\widetilde q(h_a(x)\mid a)\bigr|\,dx\,d\nu(a)=0.
\]
Exhausting $\Theta$ proves that the right side of
\cref{eq:conditional-rice} is independent, as an $L^1_{\mathrm{loc}}$
equivalence class, of the chosen version.

\subsection{Joint-density formula}
\label{app:joint-formula}

On every bounded box, the map
\[
 T(x,a)=\left(h_a(x),a\right)
\]
is Lipschitz, with Jacobian
\[
 J_T(x,a)=\left|dx^{d-1}+\sum_{j=1}^{d-1}ja_jx^{j-1}\right|.
\]
The multidimensional area formula gives
\begin{align*}
 &\int_{B\times\R^{d-1}} f(T(x,a))J_T(x,a)\,dx\,da\\
 &\qquad=\int_{\R^d} f(\alpha)
 \#\{(x,a)\in B\times\R^{d-1}:T(x,a)=\alpha\}\,d\alpha\\
 &\qquad=\int_{\R^d} f(\alpha)Z_\alpha(B)\,d\alpha.
\end{align*}
Apply the formula first on bounded boxes and then use monotone convergence.
Tonelli identifies the left side with the integral over $B$ of
\cref{eq:joint-rice}.  The area formula counts geometric preimages, as required
for distinct roots.  Applying the same formula to
$|f-\widetilde f|$ shows that two density versions give the same
$L^1_{\mathrm{loc}}$ class on the right-hand side.  This completes the proof of
Proposition~\ref{prop:joint-rice}.

\subsection{Recovery of the bounded-joint-density theorem}
\label{app:joint-recovery-proof}

\begin{proof}[Proof of \cref{cor:joint-density-recovery}]
By \cref{thm:intensity}, it suffices to bound $\rho_\mu$ by the coefficient in
\cref{eq:joint-density-recovery}.  Choose the density representative with
$0\le f\le\kappa$ everywhere and $f=0$ off $[-R,R]^d$.  If $|x|\le1$, the integrand in
\cref{eq:joint-rice} is at most
\[
 \kappa\left(d+R\sum_{j=1}^{d-1}j\right),
\]
and the integration domain for $(a_1,\ldots,a_{d-1})$ has volume at most
$(2R)^{d-1}$.

For $|x|>1$, put $y=x^{-1}$ and reverse the polynomial:
\[
 Q_\alpha(y)=y^dP_\alpha(1/y)
 =1+\sum_{k=1}^{d}b_ky^k,
 \qquad b_k=\alpha_{d-k}.
\]
A coordinate permutation preserves the density bound $\kappa$ for $b$.
A root at $x$ corresponds to a root at $y$.  Writing $\rho_Q$ for the expected
distinct-root density of the reversed family, the intensities transform as
$\rho_P(x)=|x|^{-2}\rho_Q(y)=y^2\rho_Q(y)$.  In the coefficient-space area
formula for $Q$, eliminate $b_1$ rather than $b_d$.  The corresponding map is
\[
 (y,b_2,\ldots,b_d)
 \longmapsto
 \left(-y^{-1}-\sum_{k=2}^d b_ky^{k-1},b_2,\ldots,b_d\right),
\]
whose Jacobian on the root graph is
\[
 \left|\frac{db_1}{dy}\right|=\frac{|Q_\alpha'(y)|}{|y|}.
\]
Moreover,
\[
 \begin{aligned}
 b_1&=-y^{-1}-\sum_{k=2}^d b_ky^{k-1},\\
 \frac{|Q_\alpha'(y)|}{|y|}
 &\le |y|^{-2}
 +R\sum_{k=2}^d(k-1)|y|^{k-2}.
 \end{aligned}
\]
Multiplying by $y^2$, using $|y|<1$, and integrating the remaining
$d-1$ coefficients over their cube gives
\[
 \begin{aligned}
 \rho_P(x)
 &\le\kappa(2R)^{d-1}
 \left(1+R\sum_{k=2}^d(k-1)\right)\\
 &\le\kappa(2R)^{d-1}
 \left(d+\frac{Rd(d-1)}2\right).
 \end{aligned}
\]
The two regions cover $\R$ and prove the claim.
\end{proof}

\subsection{The one-chart conditional-density bound}
\label{app:conditional-bound-proof}

\begin{proof}[Proof of \cref{thm:conditional-bound}]
Condition on $A_-=a$.  A root in $I\in\calI(\Theta)$ is equivalent to
$\alpha_0\in h_a(I)$.  The continuous image $h_a(I)$ is an interval, and its
length is at most its total variation on $I$.  Since $|a_j|\le R$ and
$|x|\le B$,
\[
 \begin{aligned}
 \Leb(h_a(I))
 &\le \int_I|h_a'(x)|\,dx\\
 &\le\left(dB^{d-1}+R\sum_{j=1}^{d-1}jB^{j-1}\right)|I|.
 \end{aligned}
\]
The essential density bound in
\cref{eq:conditional-essential-bound} bounds the conditional probability of
$h_a(I)$ by $K\Leb(h_a(I))$.  Averaging over $a$ proves
\cref{eq:conditional-polynomial}; summing $1+\cdots+(d-1)$ gives
\cref{eq:conditional-unit}.
\end{proof}

\subsection{The global two-chart bound}
\label{app:two-chart-proof}

\begin{proof}[Proof of \cref{thm:two-chart}]
On the inner chart $|x|\le1$, condition on
$(\alpha_1,\ldots,\alpha_{d-1})=a$ and eliminate $\alpha_0$ through
$\alpha_0=h_a(x)$ from \cref{eq:conditional-map}.  The calculation in
\cref{thm:conditional-bound} gives
\begin{equation}
 |h_a'(x)|\le d+R\sum_{j=1}^{d-1}j
 =d+\frac{Rd(d-1)}2.
 \label{eq:inner-chart-derivative}
\end{equation}
Thus every interval $J\subseteq[-1,1]$ is hit with probability at most the
first constant in \cref{eq:two-chart-constant} times $|J|$.

On either outer component $|x|\ge1$, condition instead on
$b=(\alpha_0,\ldots,\alpha_{d-2})$ and divide the root equation by
$x^{d-1}$.  It becomes
\begin{equation}
 \alpha_{d-1}=k_b(x)
 =-x-\sum_{j=0}^{d-2}b_jx^{j-d+1}.
 \label{eq:outer-chart-map}
\end{equation}
For $|x|\ge1$,
\begin{align}
 |k_b'(x)|
 &\le1+R\sum_{j=0}^{d-2}(d-1-j)|x|^{j-d}\notag\\
 &\le1+\frac{Rd(d-1)}2.
 \label{eq:outer-chart-derivative}
\end{align}
The conditional probability that $\alpha_{d-1}$ lies in $k_b(J)$ is at most
$K_{d-1}\Leb(k_b(J))$, and the image length is at most the total variation in
\cref{eq:outer-chart-derivative}.  Averaging over $b$ proves the outer-chart
bound.

Finally, split $I$ at $-1$ and $1$.  The lengths of the at most three
nondegenerate pieces sum to $|I|$.  For $\sigma\in\{-1,1\}$,
\[
 P_\alpha(\sigma)=0
 \quad\Longleftrightarrow\quad
 \alpha_0=-\sigma^d-\sum_{j=1}^{d-1}\alpha_j\sigma^j.
\]
Conditioning on the nonconstant coefficients, the right side specifies one
value of $\alpha_0$ and therefore has probability zero.  The overlapping split
points may thus be discarded.  Apply the inner or outer estimate to each
remaining piece, take a union bound, and use the larger chart constant.
\end{proof}

\section{Affine recovery, dispersion, and boundary cases}
\label{app:affine-dispersion}

\subsection{Exact affine incidence}
\label{app:exact-affine-proof}

\begin{proof}[Proof of \cref{thm:exact-affine}]
Set $q(a,x)=f_0(x)+\langle a,F(x)\rangle$.  On the incidence hypersurface
$M_I=\{(a,x)\in\operatorname{int}K\times\operatorname{int}I:q(a,x)=0\}$,
\[
 \nabla q(a,x)
 =(F(x),f_0'(x)+\langle a,F'(x)\rangle).
\]
Repeating the projection and slicing calculation in
\cref{eq:incidence-area-proof} gives
\[
 \operatorname{vol}_N\{a\in K:\exists x\in I,\ q(a,x)=0\}
 \le\int_I\int_{H_x^\circ}
 \frac{|f_0'(x)+\langle a,F'(x)\rangle|}
      {\lVert F(x)\rVert_2}\,
 d\mathcal H^{N-1}(a)\,dx.
\]
The density cap proves
$C_{K,A}^{\rm aff}\le A\Lambda_K^{\rm aff}$.

For the reverse inequality, fix an interior $x$ for which
$H_x^\circ\ne\varnothing$ and let
$D_h=\{a\in K:q(a,x)q(a,x+h)<0\}$.  Coarea in the normal direction to the
affine hyperplane $q(a,x)=0$, followed by the same dominated-convergence
argument as in \cref{eq:first-variation-coarea}, yields
\[
 \lim_{h\to0}\frac{\operatorname{vol}_N(D_h)}
 {|h|\operatorname{vol}_N(K)}
 =\lambda_K^{\rm aff}(f_0,F;x).
\]
Every point of $D_h$ produces a root.  Applying the capped density
\cref{eq:capped-extremal-density}, with this $D_h$, multiplies its
probability by $A$ for all sufficiently small $h$.  Letting $h\to0$ and
taking the essential supremum proves the reverse inequality.
Points with $H_x^\circ=\varnothing$ have zero incidence speed by definition
and require no lower estimate.  Restricting to $\operatorname{int}K$ and
$\operatorname{int}I$ loses no root probability: every law in
$\calD_{K,A}$ is absolutely continuous, while $\partial K$ and the two
endpoint hyperplanes are Lebesgue-null.
\end{proof}

\subsection{The affine cube-slab bound}
\label{app:affine-proof}

\begin{proof}[Proof of \cref{thm:affine}]
At a root $x_*$, $\langle a,g(x_*)\rangle=c(x_*)$.  Fix $x_0\in I$.  The root
therefore implies
\begin{align*}
 &|\langle a,g(x_0)\rangle-c(x_0)|\\
 &\quad\le |\langle a,g(x_0)-g(x_*)\rangle|
      +|c(x_*)-c(x_0)|\\
 &\quad\le (R\sqrt N L_g+L_c)|I|.
\end{align*}
Apply \cref{lem:cube-slab} with $u=g(x_0)$, $b=c(x_0)$, and the displayed
half-width.
\end{proof}

\subsection{Global bounds for the monomial feature curve}
\label{app:monomial-bounds}

Let $v(x)=(1,x,\ldots,x^{d-1})$ and $s(x)=\lVert v(x)\rVert_2$.  Since
\[
 \left\lVert\left(\frac{v}{s}\right)'\right\rVert_2
 \le\frac{\lVert v'\rVert_2}{s},
\]
it suffices to bound the ratio on the right.  For $|x|\le1$,
\begin{align*}
 \lVert v'(x)\rVert_2^2
 &=\sum_{j=1}^{d-1}j^2|x|^{2j-2}\\
 &\le(d-1)^2\sum_{k=0}^{d-2}|x|^{2k}
 \le(d-1)^2s(x)^2.
\end{align*}
For $|x|\ge1$,
\begin{align*}
 \lVert v'(x)\rVert_2^2
 &\le\frac{(d-1)^2}{|x|^2}
       \sum_{j=1}^{d-1}|x|^{2j}
 \le\frac{(d-1)^2}{|x|^2}s(x)^2.
\end{align*}
This proves the first inequality in \cref{eq:monomial-lipschitz}, as well as the
stronger exterior estimate $\lVert v'\rVert_2/s\le(d-1)/|x|$.

Now put $r(x)=x^d/s(x)$.  Since $|s'|\le\lVert v'\rVert_2$,
\[
 |r'(x)|
 \le\frac{d|x|^{d-1}}{s(x)}
 +\frac{|x|^d}{s(x)}\frac{\lVert v'(x)\rVert_2}{s(x)}.
\]
When $|x|\le1$, the two terms are at most $d$ and $d-1$.  When
$|x|\ge1$, use $s(x)\ge|x|^{d-1}$ and the stronger exterior estimate to obtain
the same two bounds.  Therefore $|r'|\le2d-1$ globally.  Replacing $r$ by
$-r$ gives the required bound for the offset $c$ in
\cref{thm:affine}.

\subsection{The dispersion interface}
\label{app:dispersion-proof}

\begin{lemma}[One-dimensional uniformization]
\label{lem:dispersion-uniformization}
Let $\Theta$ be a compact interval and let
$\ell_1,\ldots,\ell_T:\Theta\to\R$ be independent piecewise
$L$-Lipschitz functions, each with at most $K$ discontinuities.  For
$\varepsilon>0$ and $\rho\in\Theta$, let
\[
 D_\ell(T,\varepsilon,\rho)
 =\#\{t:\ell_t\text{ is not $L$-Lipschitz on }
 \Theta\cap[\rho-\varepsilon,\rho+\varepsilon]\}.
\]
Then
\[
 \E\max_{\rho\in\Theta}D_\ell(T,\varepsilon,\rho)
 \le \max_{\rho\in\Theta}\E D_\ell(T,\varepsilon,\rho)
 +O\!\left(\sqrt{T\log(TK)}\right).
\]
\end{lemma}

This is the compact-domain restriction of Theorem~7 of
\citet{balcan2020semibandit}; its interval-class uniformization proof is
unchanged after intersecting every interval with $\Theta$.  It uses
independence and the discontinuity bound, but no bounded-utility assumption.

\begin{proof}[Proof of \cref{cor:dispersion-interface}]
For every fixed $\rho$, \cref{eq:expected-boundary-count} with
$|I|=2\varepsilon$ gives
$\E D(T,\varepsilon,\rho)\le2mCT\varepsilon$; intersecting the interval with
$\Theta$ can only shorten it.  Applying
\cref{lem:dispersion-uniformization} gives
\[
 \begin{aligned}
 \E\max_\rho D(T,\varepsilon,\rho)
 &\le\max_\rho\E D(T,\varepsilon,\rho)\\
 &\quad+O\!\left(\sqrt{T\log(TK)}\right).
 \end{aligned}
\]
This proves \cref{eq:uniform-boundary-count}.  Since
$\sqrt T\le T\varepsilon$ for $\varepsilon\ge T^{-1/2}$, the dispersion claim
follows.

For the regret claim, let $\mathcal G_h$ be a grid of covering radius at most
$h$ and cardinality at most $W/h+2$.  For any $\rho\in\Theta$, choose
$g\in\mathcal G_h$ with $|g-\rho|\le h$.  On a round that is $L$-Lipschitz
throughout $\Theta\cap[\rho-h,\rho+h]$, one has
$u_t(\rho)-u_t(g)\le Lh$; on every other round the difference is at most one.
Hence, also for an approximating sequence if the supremum is not attained,
\[
 \sup_{\rho\in\Theta}\sum_{t=1}^T u_t(\rho)
 -\max_{g\in\mathcal G_h}\sum_{t=1}^T u_t(g)
 \le LTh+\max_{\rho\in\Theta}D(T,h,\rho).
\]
Take expectations, invoke \cref{eq:uniform-boundary-count}, and add the
finite-expert reward bound $\sqrt{2T\log|\mathcal G_h|}$.  This proves
\cref{eq:generic-grid-regret}.
\end{proof}

\subsection{Boundary cases and consistency checks}
\label{app:boundary-cases}

\paragraph{Degree one.}
For $d=1$, $P_{\alpha_0}(x)=x+\alpha_0$ has the single root $-\alpha_0$.
The hitting constant is exactly the interval concentration constant of that
random variable.  It is finite if and only if the root law has a bounded
density, in exact agreement with \cref{thm:intensity}.  The conditional bound
reduces to $c_\mu\le K$.

\paragraph{Atoms and repeated roots.}
If a coefficient law gives positive probability to a root at a deterministic
location $x_*$, the expected root measure has an atom and the hitting constant
is infinite.  Repeated roots are counted once.  The elementary sandwich
\cref{eq:root-count-sandwich} remains valid for every coefficient vector,
including discriminant-zero ones.  Absolute continuity is needed only when a
coefficient-space formula is represented by integration against a density; the
arbitrary-law characterization in \cref{thm:intensity} needs no absolute
continuity.

\paragraph{Finite hitting without conditional density.}
If $U\sim\operatorname{Unif}[1,2]$ and
$P_U(x)=(x-U)(x-U-1)$, the coefficient law lies on a curve and the conditional
law of the constant coefficient given the linear coefficient is a point mass.
Nevertheless, the roots are $U$ and $U+1$, so their aggregate density is
$\ind_{[1,2]}+\ind_{[2,3]}$ up to endpoint values and
\cref{thm:intensity} gives $c_\mu<\infty$.

\section{Proofs of the graph-learning applications}
\label{app:graph-applications}

\subsection{Cost-sensitive Gaussian-RBF graph SSL}
\label{app:rbf-application-proof}

\begin{lemma}[Zero count for exponential polynomials]
\label{lem:exp-poly-zeros}
Let $\beta_1<\cdots<\beta_M$ be real, and let
\[
 h(x)=\sum_{j=1}^M c_j e^{-\beta_jx}
\]
be nonzero.  Then $h$ has at most $M-1$ distinct real zeros.
\end{lemma}

\begin{proof}
The claim follows by induction on $M$.  It is immediate for $M=1$.
For $M\ge2$, multiplication by $e^{\beta_1x}$ does not change the zero set.
Writing
\[
 g(x)=e^{\beta_1x}h(x)
     =c_1+\sum_{j=2}^M c_j e^{-(\beta_j-\beta_1)x},
\]
Rolle's theorem shows that any $k$ distinct zeros of $g$ produce at least
$k-1$ distinct zeros of $g'$.  If $g'\equiv0$, then $g$ is a nonzero constant
and has no zeros.  Otherwise, after removing zero coefficients, $g'$ is an
exponential polynomial with at most $M-1$ distinct exponents.  The induction
hypothesis gives at most $M-2$ distinct zeros of $g'$, and hence
$k\le M-1$.
\end{proof}

Fix a round and suppress its index.  Write $r=|U|$ and
$s(\gamma)=s_U(\gamma)$.  Since every edge weight is at least
$\omega=e^{-BD}$, for every $x\in\R^r$,
\begin{equation}
 \begin{aligned}
 x^\top Q_\gamma x
 &=\sum_{\{u,v\}\subseteq U}w_\gamma(u,v)(x_u-x_v)^2\\
 &\quad+\sum_{u\in U,\,j\in L}w_\gamma(u,j)x_u^2
 \ge \ell\omega\lVert x\rVert_2^2.
 \end{aligned}
 \label{eq:rbf-grounded-energy}
\end{equation}
Thus $Q_\gamma\succ0$.  A coordinatewise argument gives the stronger norm
estimate needed below.  If $|x_u|=\lVert x\rVert_\infty$, then
\[
 \operatorname{sign}(x_u)(Q_\gamma x)_u
 \ge\left(\sum_{j\in L}w_\gamma(u,j)\right)
       \lVert x\rVert_\infty
 \ge\ell\omega\lVert x\rVert_\infty.
\]
Consequently,
\begin{equation}
 \lVert Q_\gamma^{-1}\rVert_\infty\le(\ell\omega)^{-1}.
 \label{eq:rbf-inverse-infinity}
\end{equation}

The discrete maximum principle applied to the harmonic equations gives
$0\le s_v(\gamma)\le1$ for every $v\in U$.  Differentiating
$Q_\gamma s(\gamma)=b_\gamma$ yields
\[
 s'(\gamma)=Q_\gamma^{-1}(b_\gamma'-Q_\gamma's(\gamma)).
\]
After adjoining the fixed boundary scores $s_j=y_j$ for $j\in L$, the
$u$th coordinate of the right-hand vector is
\begin{equation}
 (b_\gamma'-Q_\gamma's)_u
 =\sum_{a\in X\setminus\{u\}}
   w_\gamma'(u,a)(s_a-s_u).
 \label{eq:rbf-differentiated-harmonic}
\end{equation}
Because $|w_\gamma'(u,a)|=d(u,a)w_\gamma(u,a)\le D$ and all scores lie in
$[0,1]$, \cref{eq:rbf-inverse-infinity,eq:rbf-differentiated-harmonic}
give
\[
 \lVert s'(\gamma)\rVert_\infty
 \le\frac{(n-1)D}{\ell\omega},
\]
which proves \cref{eq:rbf-score-speed}.

The section calculation in
\cref{eq:rbf-exact-incidence-speed} is explicit.  Put
$p=s_v(\gamma)$ and parameterize the line
$F_v(\gamma)^\perp$ by
\[
 a(z)=(-(1-p)z,pz),\qquad
 |z|\le\frac{R}{\max\{p,1-p\}}.
\]
Since $d\mathcal H^1(a)=\lVert F_v(\gamma)\rVert_2,dz$ and
$|\langle a(z),F_v'(\gamma)\rangle|=|z s_v'(\gamma)|$,
\cref{eq:body-incidence-speed} gives
\[
 \lambda_{Q_R}(F_v;\gamma)
 =\frac{|s_v'(\gamma)|}{4\max\{p,1-p\}^2}
 \le |s_v'(\gamma)|.
\]
The law of $(C^{\rm FN},-C^{\rm FP})$ belongs to
$\calD^2_{R,\kappa}=\calD_{Q_R,A}$.  The upper half of
\cref{thm:body-characterization} and the preceding derivative estimate prove
\cref{eq:rbf-root-hitting}.

It remains to verify boundary complexity without any discretization of the
dissimilarities.  Every entry of $Q_\gamma$ and $b_\gamma$ is a sum of at most
$n$ functions $e^{-\beta\gamma}$ with $\beta\in[0,D]$.  With
\[
 \Delta(\gamma)=\det Q_\gamma,\qquad
 N_v(\gamma)=e_v^\top\operatorname{adj}(Q_\gamma)b_\gamma,
\]
the determinant expansion gives at most $r!n^r$ exponential terms in each
function before terms with the same exponent are collected.  Moreover,
\cref{eq:rbf-grounded-energy} implies $\Delta(\gamma)>0$ on $[0,B]$.
Hence
\[
 C^{\rm FN}s_v(\gamma)-C^{\rm FP}(1-s_v(\gamma))
 =\frac{H_v(\gamma)}{\Delta(\gamma)},
 \qquad
 H_v=(C^{\rm FN}+C^{\rm FP})N_v-C^{\rm FP}\Delta.
\]
After collecting equal exponents, $H_v$ has at most $2r!n^r$ distinct
exponential terms.  If $H_v$ is nonzero, \cref{lem:exp-poly-zeros} gives at
most $2r!n^r-1$ distinct real zeros.
If $H_v$ is identically zero, the prediction is fixed by the tie rule.
Consequently the round has fewer than
\[
 2r\,r!n^r\le 2n\,n!n^n=K_n
\]
discontinuities.  For integer dissimilarities, the change of variable
$z=e^{-\gamma}$ improves the bound to $D|U|^2\le Dn^2$.

The functions $u_t$ depend on independent cost pairs and fixed base
instances, so they are independent across rounds.  Apply
\cref{cor:dispersion-interface} with $m\le n$,
$C=A(n-1)D/(\ell\omega)$, $K=K_n$, and $L=0$ to obtain
\cref{eq:rbf-dispersion}.  Finally, use the grid
\[
 \{0,T^{-1/2},2T^{-1/2},\ldots,
   \lfloor B\sqrt T\rfloor T^{-1/2}\}\cup\{B\}.
\]
The nearest-grid comparator can differ from the best continuous parameter
only on rounds with a discontinuity in an interval of radius $T^{-1/2}$.
Combining \cref{eq:rbf-dispersion} with the standard exponential-weights
bound $\sqrt{2T\log(B\sqrt T+2)}$ proves \cref{eq:rbf-regret}.
A direct implementation solves one grounded system at each grid point and
uses $O((B\sqrt T+2)n^3)$ arithmetic operations per round.

\subsection{Common-offset polynomial-kernel graph SSL}
\label{app:ssl-application-proof}

\subsubsection{Harmonic extension and transition polynomials}
\label{app:ssl-transitions}

Fix a round and suppress the subscript $t$.  Write $r=|U|$ and let
$\gamma\in[-\eta,B+\eta]$.  Since $s(a,b)\ge s_->\eta$, every weight
$(s(a,b)+\gamma)^q$ is strictly positive.  For $z\in\R^r$,
\begin{equation}
 \begin{aligned}
 z^\top Q(\gamma)z
 &=\sum_{\{a,b\}\subseteq U}w_\gamma(a,b)(z_a-z_b)^2\\
 &\quad+\sum_{a\in U,\,\ell\in L}w_\gamma(a,\ell)z_a^2.
 \end{aligned}
 \label{eq:ssl-grounded-laplacian}
\end{equation}
The second sum is positive for every nonzero $z$, because $L$ is nonempty and
the graph is complete.  Thus $Q(\gamma)$ is positive definite and
$\det Q(\gamma)>0$ throughout the parameter range.

The first-order conditions for the quadratic energy with $f_L=y_L$ are
$Q(\gamma)f_U=b(\gamma)$, proving
\cref{eq:ssl-harmonic-solution}.  Cramer's rule now gives, for each $a\in U$,
\begin{equation}
 2f_U(\gamma)_a-1
 =\frac{2e_a^\top\operatorname{adj}(Q(\gamma))b(\gamma)
          -\det Q(\gamma)}{\det Q(\gamma)}
 =\frac{\psi_a(\gamma)}{\det Q(\gamma)}.
 \label{eq:ssl-score-sign}
\end{equation}
Every entry of $Q(\gamma)$ and $b(\gamma)$ is a polynomial of degree at most
$q$.  Hence $\det Q(\gamma)$ has degree at most $qr$, while every term in
$\operatorname{adj}(Q(\gamma))b(\gamma)$ has degree at most
$q(r-1)+q=qr$.  Therefore
\begin{equation}
 \deg\psi_a\le qr.
 \label{eq:ssl-transition-degree}
\end{equation}

If $\psi_a$ is the zero polynomial, \cref{eq:ssl-score-sign} implies
$f_U(\gamma)_a=1/2$ throughout the relevant range, and the fixed tie rule
makes its prediction constant.  If $\psi_a$ is nonzero, the prediction can
change only at one of its at most $qr$ distinct real roots.  The vector of all
$r$ predictions, and therefore the accuracy utility, is constant on every
component of the complement of these roots.  Thus all discontinuities are
contained among the roots of the at most $r$ nonzero polynomials
$\psi_a(\lambda+V)$; identically zero polynomials are discarded from the
active transition set.  Their total number is at most
\begin{equation}
 r(qr)=qr^2\le qn^2.
 \label{eq:ssl-boundary-count}
\end{equation}

\subsubsection{Root anti-concentration and dispersion}
\label{app:ssl-anticoncentration}

For a nonzero $\psi_a$, let $z_1<\cdots<z_M$ be its distinct real roots.
These numbers are fixed by the latent instance, and
$M\le qr$.  For every interval $I\subseteq[0,B]$,
\begin{align}
 &\Prob\bigl(\exists\lambda\in I:\psi_a(\lambda+V)=0\bigr)\notag\\
 &\quad\le\sum_{j=1}^M\Prob(V\in z_j-I)\notag\\
 &\quad\le M\kappa|I|
 \le qr\kappa|I|.
 \label{eq:ssl-root-hitting-proof}
\end{align}
Here $z_j-I=\{z_j-x:x\in I\}$ is an interval of length $|I|$.  Equivalently,
on the whole real line the ordered roots in the parameter coordinate are
$z_j-V$, whose laws have densities bounded by $\kappa$; the display then
restricts their hitting events to $I\subseteq[0,B]$.  This proves
\cref{eq:ssl-root-hitting} directly in the canonical root coordinates of
\cref{thm:root-coordinates}.

For every round there are at most $m=n$ transition equations, each satisfying
the common estimate $C|I|$ with $C=qn\kappa$, and there are at most
$K=qn^2$ discontinuities.  Each utility is a measurable function of its own
offset $V_t$ and a fixed latent instance.  The independence of the offsets
therefore makes $u_1,\ldots,u_T$ independent.  Applying
\cref{cor:dispersion-interface} with $L=0$ gives
\[
 \E\max_\rho D(T,\varepsilon,\rho)
 \le2qn^2\kappa T\varepsilon
   +O\!\left(\sqrt{T\log(Tqn^2)}\right),
\]
which is \cref{eq:ssl-dispersion}.

\subsubsection{A full-information algorithm and its regret}
\label{app:ssl-regret}

Let $h=T^{-1/2}$ and form the grid
\[
 \mathcal G_h
 =\{0,h,2h,\ldots,\lfloor B/h\rfloor h\}\cup\{B\}.
\]
It has at most $B\sqrt T+2$ points, and every point of $[0,B]$ is within
distance $h$ of a grid point.  Run the exponential-weights forecaster on the
experts $g\in\mathcal G_h$, using the full utility vector
$(u_t(g))_{g\in\mathcal G_h}$ revealed after each round.  The standard
potential argument for rewards in $[0,1]$, with the learning rate optimized,
gives
\begin{equation}
 \E_{\rm alg}\!\left[
  \max_{g\in\mathcal G_h}\sum_{t=1}^T u_t(g)
  -\sum_{t=1}^T u_t(\lambda_t)
 \right]
 \le\sqrt{2T\log|\mathcal G_h|}.
 \label{eq:ssl-hedge-bound}
\end{equation}

For a realized sequence, choose
$\lambda^*\in\arg\max_{\lambda\in[0,B]}\sum_tu_t(\lambda)$ and a grid point
$g^*$ with $|g^*-\lambda^*|\le h$.  If $u_t$ has no discontinuity within
$[\lambda^*-h,\lambda^*+h]$, piecewise constancy gives
$u_t(g^*)=u_t(\lambda^*)$.  Since utilities lie in $[0,1]$,
\begin{equation}
 \sum_{t=1}^T\bigl(u_t(\lambda^*)-u_t(g^*)\bigr)
 \le\max_\rho D(T,h,\rho).
 \label{eq:ssl-grid-approximation}
\end{equation}
The maximum exists because the sum has finitely many constant pieces and the
tie values are fixed.

Take expectations in \cref{eq:ssl-grid-approximation}, substitute
$h=T^{-1/2}$ into \cref{eq:ssl-dispersion}, and add
\cref{eq:ssl-hedge-bound}.  Since
$|\mathcal G_h|\le B\sqrt T+2$, this gives exactly
\cref{eq:ssl-regret}.  A direct implementation evaluates the harmonic
extension at every grid point using one linear solve, and hence uses
$O((B\sqrt T+2)n^3)$ arithmetic operations per round without exploiting any
additional matrix structure.

\end{document}